\documentclass[11pt]{article}

\usepackage{arxiv-style}
\usepackage{amsmath, algorithm, algpseudocode, bm, bbm}
\usepackage{booktabs}           
\usepackage{bbding, pifont}     
\usepackage{multirow, multicol}

\usepackage{aliascnt}

\makeatletter

\newaliascnt{lemma}{theorem}
\newtheorem{lemma}[lemma]{Lemma}
\aliascntresetthe{lemma}

\newaliascnt{proposition}{theorem}
\newtheorem{proposition}[proposition]{Proposition}
\aliascntresetthe{proposition}
\makeatother

\usepackage{cleveref}

\crefname{lemma}{lemma}{lemmas}
\Crefname{lemma}{Lemma}{Lemmas}

\crefname{proposition}{proposition}{propositions}
\Crefname{proposition}{Proposition}{Propositions}

\newcommand{\no}{ }
\newcommand{\yes}{\ding{51}}

\newcommand{\Regret}{\mathrm{Regret}}
\newcommand{\Violation}{\mathrm{Violation}}

\DeclareMathOperator*{\argmin}{arg\,min}

\newcommand{\qstq}{\quad\mathrm{s.t.}\quad}
\newcommand{\st}{\;\mathrm{s.t.}\;}
\newcommand{\proj}{\mathrm{proj}}
\newcommand{\bigO}{{O}}
\newcommand{\tbigO}{\widetilde{{O}}}

\newcommand{\dom}{\mathrm{dom}}
\newcommand{\ridom}{\mathrm{ri\,dom}}

\newcommand{\calD}{{\mathcal{D}}}

\newcommand{\calF}{{\mathcal{F}}}

\newcommand{\calI}{{\mathcal{I}}}

\newcommand{\calX}{{\mathcal{X}}}

\newcommand{\bbE}{\mathbb{E}}

\newcommand{\bbP}{\mathbb{P}}

\newcommand{\bbR}{\mathbb{R}}

\newcommand{\bfx}{\mathbf{x}}
\newcommand{\bfy}{\mathbf{y}}

\title{
Constrained Online Convex Optimization \\
without Slater's Condition
}

\author{
Kihyun Yu$^{1}$ \quad Junehee Lee$^{2}$ \quad Dabeen Lee$^{2}$\\[0.3em]
$^{1}$KAIST \quad $^{2}$Seoul National University\\
\texttt{khyu99@kaist.ac.kr, \{leo0402,dabeenl\}@snu.ac.kr}
}

\date{}

\begin{document}
\maketitle
\begin{abstract}

We study constrained online convex optimization with adversarial losses and stochastic or adversarial constraints. For stochastic constraints, existing algorithms that achieve nearly optimal regret and constraint violation bounds typically rely on regularity assumptions such as Slater's condition, while adversarial-constraint algorithms avoid these assumptions by using a rather restrictive round-wise feasible comparator. We bridge this gap with an anytime primal-dual framework that incorporates an adaptive regularizer into the dual update. The regularizer stabilizes the dual process without relying on the negative drift induced by Slater's condition. For stochastic constraints and convex losses, our algorithm achieves $O(\sqrt{T})$ expected regret and $O(\sqrt{T}\log T)$ expected cumulative constraint violation. Furthermore, we show that our algorithm also admits high-probability bounds of the same order on regret and constraint violation. For strongly convex losses, the regret bound improves to $O(\log T)$ with a violation bound of the same order. 
With a minor modification, the framework also applies to adversarial constraints and provides guarantees for hard constraint violation.

\end{abstract}

\section{Introduction}
Constrained Online Convex Optimization (COCO) is a sequential decision-making problem under uncertainty in which, at each round, a decision maker selects an action without prior knowledge of the convex loss function, while feasibility requirements restrict the decision-making process. The goal is to minimize the cumulative loss while satisfying these requirements over time. COCO arises naturally in a wide variety of real-world applications, including online advertising with budget constraints~\citep{mehta2007adwords}, portfolio optimization with risk constraints~\citep{cover1991universal, rockafellar2000optimization}, wireless network resource management under power constraints~\citep{tse2005wireless}, and online recommendation systems with fairness constraints~\citep{singh2018fairness}.

The COCO problem is typically formulated as follows (a detailed formulation is provided in \Cref{sec:preliminaries}). In each round $t$, the decision maker selects an action $\bfx_t\in\mathcal{X}$. Then, the loss function $f_t$ and the constraint function $g_t$ are revealed to the decision maker, where feasibility is characterized by the functional constraint $g_t(\bfx) \leq 0$. Given a comparator $\bfx^\star\in \mathcal{X}$, the performance of the decision maker is measured by the regret and cumulative constraint violation over $T$ rounds, denoted by $\Regret(T)$ and $\Violation(T)$, respectively:
\begin{align*}
    &\Regret(T) = \sum_{t=1}^T (f_t(\bfx_t) - f_t(\bfx^\star)),
    \quad
    \Violation(T) = \sum_{t=1}^T g_t(\bfx_t).
\end{align*}

COCO has been studied under various settings. One line of work considers COCO with \emph{stochastic constraints}~\citep{
yu2017online, wei2020online}. In this setting, the constraint functions are sampled i.i.d. from an underlying distribution, whereas the loss functions may vary arbitrarily over time. \citet{yu2017online} proposed a drift-plus-penalty algorithm that achieves $\bigO(\sqrt{T})$ bounds on regret and cumulative constraint violation. However, their analysis relies on Slater's condition, which assumes the existence of a strictly feasible decision in expectation, i.e., there exist $\hat{\bfx}$ and $\varepsilon>0$ such that $\bbE[g_1(\hat{\bfx})]=\cdots=\bbE[g_T(\hat{\bfx})]\leq -\varepsilon$, where the expectation is taken over the randomness of the constraint functions. More importantly, their bound on cumulative constraint violation is $O(\varepsilon^{-1}\sqrt{T})$, which deteriorates as $\varepsilon$ approaches zero.
This motivates us to investigate whether this condition can be relaxed while maintaining regret and constraint violation guarantees.

\citet{wei2020online} made some progress in this direction for the same setting by replacing Slater's condition with what they call the sequential existence of Lagrange multipliers (SELM) condition, which assumes the existence of uniformly bounded Lagrange multipliers. Since Slater's condition implies SELM, the latter is a weaker assumption. Nevertheless, their analysis still relies on the boundedness of the Lagrange multipliers, with the hidden constant in the constraint violation bound depending on the corresponding multiplier bound. Consequently, whether these guarantees can be achieved without such regularity assumptions has remained unresolved for COCO with stochastic constraints.

In contrast to the stochastic constraint setting, one can remove the necessity of a regularity assumption in another line of settings, called COCO with \emph{adversarial constraints} against a \emph{restrictive} comparator, where the constraint functions as well as the loss functions can vary arbitrarily over time~\citep{yi2020distributed, sinha2024optimal}. Here, the comparator is restrictive in that it is defined to be \emph{round-wise feasible}: 
$g_t(\bfx^\star) \leq 0, \; \forall t=1,\ldots,T$. 
For this setting, \citet{yi2020distributed} proposed an algorithm without assuming Slater's condition or bounded Lagrange multipliers, but its guarantees are suboptimal. Later, \citet{sinha2024optimal} proposed an algorithm that does not rely on these conditions but achieves nearly optimal regret and constraint violation bounds of $\bigO(\sqrt{T})$ and $\tbigO(\sqrt{T})$\footnote{Notation $\tbigO(\cdot)$ is used to hide polylogarithmic factors, and the precise bound of \citet{sinha2024optimal} on constraint violation is $\bigO(\sqrt{T}\log T)$.},
respectively. The bounds are independent of any Slater constant or Lagrange multiplier bound.
Additionally, \citet{sarkar2026optimal} developed an anytime extension that does not require prior knowledge of the horizon $T$.

Given the success in the adversarial constraint setting, one may ask whether the same guarantees carry over to the stochastic constraint setting, since any realized stochastic constraint sequence can be viewed as a special case of adversarial constraint sequences. Unfortunately, the answer is no. The key distinction between the two settings lies in the definition of the comparator $\bfx^\star$, which plays a central role in the performance metrics. 
As mentioned above, the previous work for the adversarial setting takes a round-wise feasible comparator, but in the stochastic constraint setting, the comparator is defined to be \emph{feasible in expectation}: 
$\bbE[g_1(\bfx^\star)]=\cdots= \bbE[g_T(\bfx^\star)]\leq 0$. 
This distinction implies that directly applying the algorithm of \cite{sinha2024optimal} to the stochastic constraint setting can provide guarantees only under the stronger assumption that $\bfx^\star$ is feasible almost surely, i.e., $g_t(\bfx^\star) \leq 0$ with probability $1$. This yields a weaker notion of regret, which is not satisfactory for the stochastic constraint setting.

\subsection{Our Contributions}

Prior guarantees for COCO with stochastic constraints can be achieved only under additional assumptions or with respect to a weaker notion of regret. Motivated by this gap, this paper develops an algorithm for this setting that achieves nearly optimal guarantees\footnote{By nearly optimal guarantees, we mean regret and violation bounds of order $\tbigO(\sqrt{T})$.} without these limitations. We summarize the main contributions below and provide a comparison with recent works in \Cref{tab:comparison}.
\begin{itemize}
    \item 
    We propose a novel algorithm (\Cref{alg:cumulative}) for COCO with stochastic constraints that does not rely on additional assumptions, including Slater's condition and bounded Lagrange multipliers. 
    For the stochastic constraint setting, we show that our algorithm achieves expected regret and constraint violation bounds of $\bigO(\sqrt{T})$ and $\tbigO(\sqrt{T})$,
    respectively (\Cref{thm:stoc + exp + cvx}).
    The novelty of our algorithm stems from incorporating an adaptive regularizer into the dual update. In principle, this step stabilizes the dual process without relying on the negative drift induced by Slater's condition. More technically, our regularizer is carefully designed to offset an accumulated term involving the derivative of a Lyapunov function that arises in the regret analysis. 

    \item 
    Beyond expected guarantees, we also provide high-probability guarantees. In particular, given a confidence parameter $\delta$, with probability at least $1-\delta$, our algorithm achieves regret and violation bounds of order $\bigO(\sqrt{T})$ and  $\tbigO(\sqrt{T})$, respectively  (\Cref{thm:stoc + hp + cvx}). The key step is to obtain an upper bound on the Lyapunov-weighted sum of comparator constraints by applying a concentration inequality for supermartingales. We then modify the regularizer appropriately so that it can offset this bound as well.

    \item 
    We further extend the guarantees to strongly convex losses and adversarial constraints. Under stochastic constraints, when losses are strongly convex, our algorithm improves the expected and high-probability regret bounds to $\bigO(\log T)$ and $\bigO(\log T + \sqrt{\log(1/\delta)})$, respectively, while the constraint violation bounds remain the same as in the convex loss case (Theorems~\ref{thm:stoc + exp + strcvx} and~\ref{thm:stoc + hp + strcvx}). 
    
    \item Moreover, under adversarial constraints, a slightly modified version of our algorithm achieves regret and hard constraint violation bounds of $\bigO(\sqrt{T})$ and $\tbigO(\sqrt{T})$, respectively, for convex losses (\Cref{thm:adv + exp + cvx}), and improves the regret bound to $\bigO(\log T)$ for strongly convex losses (\Cref{thm:adv + exp + strcvx}). Here, hard constraint violation is a stronger notion of constraint violation that does not allow cancellation across rounds.

    \item 
    Our algorithm is unified and anytime. The same algorithmic framework, based on the dual update with a regularizer, applies to both stochastic and adversarial constraint settings, as well as to both convex and strongly convex losses. Moreover, the parameter choices of our algorithm do not require prior knowledge of the horizon $T$, and hence the algorithm runs in an anytime manner without using the doubling trick. 
\end{itemize}

The rest of the paper is organized as follows. \Cref{sec:related work} reviews additional related work. \Cref{sec:preliminaries} presents the problem formulations and assumptions. \Cref{sec:algorithm} introduces our anytime primal-dual algorithm, and \Cref{sec:regularizer} explains the intuition for how the regularizer stabilizes the dual process.
\Cref{sec:results} presents the main theoretical results for stochastic and adversarial constraints under convex and strongly convex losses, together with an overview of the proofs. 
Sections~\ref{sec:anal:stoc + exp + cvx} and~\ref{sec:anal:stoc + hp + cvx} provide the analyses for expected and high-probability guarantees under stochastic constraints, respectively, and \Cref{sec:anal:stoc + strcvx} extends the analysis to strongly convex losses. 
\Cref{sec:conclusion} concludes the paper, and 
the appendix contains the deferred proofs, the analysis for adversarial constraints, and auxiliary lemmas.

\begin{table}
\caption{Comparison of COCO algorithms with optimal guarantees up to logarithmic factors. The columns represent the following. \textbf{Stochastic}: the stochastic constraint setting is considered; \textbf{Exp.}, \textbf{H.P.}: expected and high-probability guarantees are provided, respectively; \textbf{Adversarial}: the adversarial constraint setting is considered; \textbf{w/o Slater}: the algorithm does not require Slater's condition, where $\triangle$ denotes that another regularity assumption, weaker than Slater's condition, is required; \textbf{S.C.}: improved guarantees can be achieved when losses are strongly convex; \textbf{Anytime}: the algorithm does not require $T$ as an input parameter. $^\dagger$: \cite{yu2017online} considered only stochastic constraints, but the same algorithm also works for adversarial constraints \citep{neely2017online}.}

\label{tab:comparison}
\begin{center}
\begin{tabular}{lccccccc}
\toprule
\multirow{2}{*}{Algorithm}
& \multicolumn{2}{c}{Stochastic} 
& \multirow{2}{*}{Adversarial} 
& \multirow{2}{*}{w/o Slater} 
& \multirow{2}{*}{S.C.} 
& \multirow{2}{*}{Anytime}
\\
\cmidrule(lr){2-3}
& Exp. & H.P.
&
&
&
&
\\
\midrule

\cite{yu2017online}
& \yes
& \yes
& \yes$^\dagger$
& \no
& \no
& \no
\\

\cite{wei2020online}
& \yes
& \no
& \no
& $\triangle$
& \no
& \no
\\

\cite{sinha2024optimal}
& \no
& \no
& \yes
& \yes
& \yes
& \no
\\

\cite{sarkar2026optimal}
& \no
& \no
& \yes
& \yes
& \no
& \yes
\\

\midrule

\textbf{Ours}
& \yes
& \yes
& \yes
& \yes
& \yes
& \yes
\\
\bottomrule
\end{tabular}
\end{center}
\end{table}

\subsection{Related Work}\label{sec:related work}
Online Convex Optimization (OCO) provides a general framework for sequential decision-making with convex losses over a fixed feasible set. In the classical OCO setting,  a decision maker selects an action in each round and then observes a convex loss function, with the goal of minimizing regret against a fixed comparator in hindsight. Since the foundational work of \citet{zinkevich2003online}, OCO has been extensively studied in the online learning literature~\citep{cesa1996worst, cesa2006prediction, shalev2012online, hazan2016introduction, orabona2019modern}. 

In contrast to the classical OCO setting, COCO captures more complex feasible regions characterized by functional constraints and additionally requires the decision maker to control constraint violations over time.
A line of work studies COCO with fixed or known constraint functions. A common formulation in this line is OCO with long-term constraints, where the constraints may be violated at individual rounds but their cumulative violation should be controlled over the entire horizon. This formulation was initiated by \citet{mahdavi2012trading}. Subsequent works developed adaptive algorithms~\citep{jenatton2016adaptive}, established guarantees for cumulative squared constraint violations~\citep{yuan2018online}, and improved guarantees~\citep{yu2020low, yi2021regret}. 

Another line of work studies COCO with stochastic constraints, where the constraint functions are sampled i.i.d. from an underlying distribution while the loss functions may vary adversarially. In this literature, \citet{yu2017online} established expected and high-probability regret and constraint violation bounds, but their analysis relies on Slater's condition. Later, \citet{wei2020online} relaxed this requirement by assuming the existence of bounded Lagrange multipliers. In contrast, our work requires none of these assumptions but still achieves nearly optimal guarantees for stochastic constraints.

Recently, COCO has also been extensively studied under adversarial constraints, where the loss and constraint functions may vary adversarially over time. The early work of \citet{mannor2009online} considered online learning with a comparator that satisfies a sample path constraint $\sum_{t=1}^T g_t(x)\leq 0$. They proved that for their setting, no online algorithm can achieve sublinear bounds on both regret and constraint violation at the same time. Subsequent works then studied adversarial constraints under restrictive comparators with round-wise feasibility requirements~\citep{neely2017online, guo2022online, sinha2024optimal, lekeufack2024optimistic}. This line of research has since been extended to broader settings~\citep{wang2025doubly, sinha2026beyond, zhang2026multiconstraint, supantha2026universal}. In particular, \citet{pmlr-v235-garber24a, wang2025revisiting, sarkar2026projectionfree} proposed projection-free algorithms, \citet{yi2020distributed, yi2022regret, pan2026distributed} studied distributed settings, \citet{vaze2026sqrt} proposed instance-dependent guarantees, \citet{sarkar2026improved} improved guarantees, and \citet{sarkar2026optimal} proposed anytime algorithms. Although these results provide strong guarantees in the adversarial constraint setting, they do not directly apply to the stochastic constraint setting, where the comparator is required to be feasible only in expectation.

\section{Preliminaries}\label{sec:preliminaries}
This section presents the formulations of COCO studied in this paper and the required assumptions. Before introducing them, we first establish some basic notation used throughout the paper. Let $\bbE$ denote the expectation taken over randomness of the constraint functions, let $\bbR_+$ denote the set of nonnegative real numbers, let $\|\cdot\|$ denote the Euclidean norm, and let $[\cdot]_+ = \max\{0,\cdot\}$. Moreover, for a subdifferentiable function $f:\bbR^d \to \bbR$, let $\nabla f(\bfx)$ denote a subgradient of $f$ at $\bfx$. For a positive integer $T$, let $[T] = \{1,2,\ldots, T\}$.

Let $\calX \subseteq \bbR^d$ be a compact convex set for some positive integer $d$. Let $\{f_t\}_{t\geq 1}$ and $\{g_t\}_{t\geq 1}$ denote sequences of convex loss and constraint functions, respectively, that are subdifferentiable on $\calX$. Additionally, we impose the following assumptions throughout the paper.
\begin{assumption}
There exists a constant $D > 0$ such that $\|\bfx - \bfy\|\leq D$ for all $\bfx,\bfy \in \calX$.
\end{assumption}
\begin{assumption}
There exist constants $L,G > 0$ such that $\|\nabla f_t(\bfx)\| \leq L$, $\|\nabla g_t(\bfx)\|\leq L$, and $|g_t(\bfx)|\leq G$ for all $\bfx\in\calX$ and $t\geq 1$.
\end{assumption}
The first assumption ensures that  $\calX$ is bounded, while the second assumes that the subgradients and constraint functions are bounded. Both assumptions are common in the COCO literature~\citep{yu2017online}.

We now formulate an online learning problem for COCO. In each round $t$, the decision maker selects an action $\bfx_t \in \calX$, after which the loss function $f_t$ and the constraint function $g_t$ are revealed. 
We consider two settings that differ in how the constraint functions are given: the stochastic constraint setting~\citep{yu2017online} and the adversarial constraint setting~\citep{sinha2024optimal}. 
\begin{itemize}
    \item \textbf{Stochastic Constraint Setting:} In each round $t$, $g_t$ is sampled i.i.d. from an unknown distribution $\calD$, while $f_t$ is chosen arbitrarily and independently of the samples of the constraint functions. In this case, we consider the following performance metrics, which are regret and constraint violation over $T$ rounds:
    \begin{align}\label{def:reg viol}
        &\Regret(T) = \sum_{t=1}^T (f_t(\bfx_t) - f_t(\bfx^\star)),
        \quad
        \Violation(T) = \sum_{t=1}^T g_t(\bfx_t),
    \end{align}
    where the comparator $\bfx^\star \in \calX$ is independent of $\{g_t\}_{t=1}^T$ and is defined to be feasible in expectation, i.e.,
    \begin{align}\label{def:xstar stoc}
        \bfx^\star \in \argmin_{\bfx\in\calX}\; \sum_{t=1}^T f_t(\bfx) \qstq \bbE[g_1(\bfx)]=\cdots= \bbE[g_T(\bfx)] \leq 0.
    \end{align}

    \item \textbf{Adversarial Constraint Setting:} In each round $t$, $g_t$ is chosen adversarially, just as $f_t$ is. In this case, we consider the regret measure defined in the same way as in \eqref{def:reg viol}, while the notion of hard constraint violation is considered, which is defined as 
    \begin{align*}
        \Violation_+(T) = \sum_{t=1}^T [g_t(\bfx_t)]_+.
    \end{align*} 
    Here, under the assumption that $\{\bfx\in\calX: g_t(\bfx)\leq 0,\ \forall t\in [T]\}$ is nonempty~\citep{sinha2024optimal}, $\bfx^\star$ is defined as round-wise feasible, i.e.,
    \begin{equation}\label{def:xstar adv}
        \bfx^\star \in \argmin_{\bfx\in\calX} \sum_{t=1}^T f_t(\bfx) \qstq g_t(\bfx) \leq 0, \; \forall t\in [T].
    \end{equation}
\end{itemize}

Again, it is worth comparing the definitions of $\bfx^\star$ in the two settings: feasible in expectation~\eqref{def:xstar stoc} versus round-wise feasible~\eqref{def:xstar adv}. In particular, for the adversarial constraint setting, $\bfx^\star$ is defined to be round-wise feasible, since no online algorithm can achieve sublinear bounds on both regret and constraint violation at the same time if $\bfx^\star$ is instead defined as $\argmin_{\bfx \in \calX} \sum_{t=1}^T f_t(\bfx) \st \sum_{t=1}^T g_t(\bfx) \leq 0$~\citep{mannor2009online}. Moreover, under the restrictive definition of $\bfx^\star$ given in \eqref{def:xstar adv}, we may obtain guarantees for the hard constraint violation. 
On the other hand, under the stochastic constraint setting, $\bfx^\star$ is only needed to satisfy the constraint in expectation rather than for all rounds. Therefore, the difference between two choices of $\bfx^\star$ suggests that algorithms for the adversarial constraint setting do not directly translate to the stochastic constraint setting, requiring additional techniques.

\section{Unified Algorithm for COCO}\label{sec:algorithm}
In this section, we present \Cref{alg:cumulative}, an anytime algorithm for solving COCO. We emphasize that it is a unified algorithm that does not rely on Slater's condition and achieves nearly optimal guarantees in both stochastic and adversarial constraint settings under both convex and strongly convex losses. The key principle is to incorporate an adaptive regularizer into the dual update that stabilizes the dual process without relying on the negative drift induced by Slater's condition. A more detailed discussion of the regularizer is provided in \Cref{sec:regularizer}.

\begin{algorithm}[t]
\caption{COCO without Slater's condition}
\label{alg:cumulative}
\textbf{Input:} Lipschitz constant $L$; boundedness constant for constraints $G$; diameter $D$; confidence parameter $\delta \in (0,1)$; strong convexity constant $\mu$;
auxiliary function $\Psi:\bbR_+\to\bbR$; Lyapunov parameters $\{\gamma_t\}_{t\geq1}$.
\\
\textbf{Initialize:} Set $\bfx_1 \in \calX$, $Q_1 = 0$, and $\Phi_t(x) =e^{\gamma_t x} - 1$ for all $t\geq 1$.
\begin{algorithmic}
    \For{$t=1,2,\ldots$}
    \State Take $\bfx_t$ and observe $f_t,g_t$
    \State Update the primal step size $\eta_t$
    \State Update $\bfx_{t+1} = \proj_\calX\left(\bfx_t - \eta_t(\nabla f_t(\bfx_t) + \Phi_t'(Q_t)\nabla g_t(\bfx_t))\right)$  \label{line:primal}
    \State Update $R_{t}$: \label{line:regularizer}
    \begin{equation*}
        R_t = 12\gamma_t G^2 + \frac{\Psi\left(\sum_{\tau=1}^t\Phi_\tau'(Q_\tau)^2\right)-\Psi\left(\sum_{\tau=1}^{t-1}\Phi_\tau'(Q_\tau)^2\right)}{\Phi_t'(Q_t)}
    \end{equation*}
    \State Update $Q_{t+1} = Q_{t} + g_t(\bfx_t) - R_t$ \label{line:dual}
    \EndFor
\end{algorithmic}
\end{algorithm}

Our algorithm proceeds as follows; the choices of $\Psi, \gamma_t$, and $\eta_t$ will be given in the subsequent paragraph. Initially, the decision maker selects an arbitrary point $\bfx_1\in\calX$ and sets the initial dual variable $Q_1=0$. In each round $t\geq 1$, the decision maker first plays $\bfx_t$ and then observes the loss function $f_t$ and the constraint function $g_t$. Using the current dual variable $Q_t$, inspired by the works of \cite{sinha2024optimal, sarkar2026optimal}, the algorithm constructs an exponential Lyapunov function given by $\Phi_t(x)=e^{\gamma_t x}-1$ and performs a primal update based on the surrogate subgradient $\nabla f_t(\bfx_t)+\Phi_t'(Q_t)\nabla g_t(\bfx_t)$. The exponential Lyapunov function is particularly useful compared to quadratic Lyapunov functions because its derivative is strictly positive for all $Q_t$ when $\gamma_t>0$, implying that the surrogate function $f_t+\Phi_t'(Q_t)g_t$ always remains convex. The dual variable is then updated according to the observed constraint violation together with the regularizer term $R_t$, which adaptively controls its growth. 

The key parameters of \Cref{alg:cumulative} are $\Psi, \gamma_t,$ and $\eta_t$. Indeed, these parameters do not require knowledge of $T$, hence our algorithm is anytime. The choices of $\Psi, \gamma_t,$ and $\eta_t$ are given as follows, depending on the setting under consideration.
\begin{itemize}
    \item The parameter $\gamma_t$ defines the exponential Lyapunov function $\Phi_t$, while $\Psi$ is an auxiliary function used to define $R_t$. Under the stochastic constraint setting, both $\Psi$ and $\gamma_t$ are chosen according to the type of guarantee desired, namely, expected or high-probability guarantees. The precise choices of $\Psi$ and $\gamma_t$ are given as follows: if the expected guarantees are desired, then we take
    \begin{equation}\label{eq:psi + gamma exp}
    \Psi(x)=4DL\sqrt{x},
    \qquad
    \gamma_t = \min\left\{\frac{1}{12G\sqrt{t}},\frac{1}{24DL},1\right\}.
    \end{equation}
    To obtain high-probability guarantees with confidence  $\delta\in (0,1)$, we instead take
    \begin{align}\label{eq:psi + gamma hp}
    \begin{aligned}
        &\Psi(x)=4DL\sqrt{x}+4G\sqrt{(1+x)\log\frac{(\pi^2/6)(1+\log_2(1+x))^2}{\delta}},\\
        &\gamma_t =\min\left\{\frac{1}{12G\sqrt{t}},\frac{1}{24\!\left(DL+8G\sqrt{\log(12t^2/\delta)}\right)},1\right\}.
    \end{aligned}
    \end{align}
    For the adversarial constraint setting, we use the same choices of $\Psi$ and $\gamma_t$ as in the expected bound setting with stochastic constraints. Moreover, with a minor modification, the algorithm can also provide guarantees for hard constraint violation. The modified algorithm is given in \Cref{alg:hard}, and the corresponding results are presented in \Cref{sec:results:adv}.

    \item The primal step size $\eta_t$ is used in the projected subgradient descent update and is chosen according to the properties of the loss functions. Specifically, when the loss functions are convex, we employ an AdaGrad-type step size for $\eta_t$~\citep{duchi2011adaptive}. In contrast, to deduce improved performance when the loss functions are strongly convex, we use a step size based on that of \cite{bartlett2007adaptive}, with appropriate modifications to accommodate our setting. The precise choice of $\eta_t$ is summarized as follows: if losses are convex, then we take 
    \begin{align}\label{eq:eta cvx}
    \eta_t
    =\frac{\sqrt{2}D}{2\sqrt{1+\sum_{\tau=1}^t\|\nabla f_\tau(\bfx_\tau)+\Phi_\tau'(Q_\tau)\nabla g_\tau(\bfx_\tau)\|^2}}.
    \end{align}
    If losses are $\mu$-strongly convex, then we take
    \begin{align}\label{eq:eta strcvx}
        \eta_t
        =\frac{1}{\mu t + (2L/D)\sqrt{\sum_{\tau=1}^t \Phi_\tau'(Q_\tau)^2}}.
    \end{align}
\end{itemize}

\section{The Role of the Regularizer}\label{sec:regularizer}
This section explains the role of the regularizer $R_t$. In the stochastic constraint setting, the main challenge arises from the less restrictive definition of comparator $\bfx^\star$, that is, a solution satisfying the expected constraint $\bbE[g_t(\bfx^\star)] \leq 0$. Specifically, although the algorithm proposed by \cite{sinha2024optimal} attained nearly optimal guarantees in the adversarial constraint setting, its analysis heavily relies on the fact that $g_t(\bfx^\star) \leq 0$ for all $t\geq 1$, which no longer holds in our setting. To address this issue, our technique is to incorporate an adaptive regularizer into the dual update.

\subsection{Key Mechanism}
Before explaining how the regularizer works, we first introduce the main direction of the analysis. Inspired by \cite{sinha2024optimal}, we aim to show the following key relation, which simultaneously guarantees small regret and stability of the dual process. Letting $\lesssim$ denote an informal notation used only to convey intuition,
\begin{equation}\label{eq:Role of Reg:goal}
    \Regret(T) + \Phi_{T+1}(Q_{T+1}) \lesssim \bigO(\sqrt{T}).
\end{equation}

To derive this relation in the stochastic constraint setting, we then explain why the regularizer is essential. For this purpose, we first observe that the regularizer enables us to control the Lyapunov drift more carefully, which is defined as $\Phi_{t+1}(Q_{t+1})-\Phi_{t}(Q_{t})$. Intuitively, since this quantity represents the change in the dual variable, it is desirable for it to remain small in order to stabilize the dual process. Roughly, the Lyapunov drift satisfies the following relation, suggesting that $R_t$ helps control the Lyapunov drift and thereby stabilizes the dual process. 
\begin{equation}\label{eq:Role of Reg:drift}
\Phi_{t+1}(Q_{t+1})-\Phi_{t}(Q_{t})
\lesssim \Phi_t'(Q_t)\bigl(g_t(\bfx_t)-R_t\bigr).
\end{equation}

Combining this drift analysis with the regret analysis for the surrogate loss $f_t + \Phi_t'(Q_t) g_t$ gives us
\begin{align} \label{eq:Role of Reg:I II III}
\begin{aligned}
    &\Regret(T) + \Phi_{T+1}(Q_{T+1}) \\
    &\lesssim  \bigO\left(\sqrt{T}\right)  + \underbrace{\sqrt{\sum_{t=1}^T\Phi_t'(Q_t)^2}}_{\textrm{(I)}} - \underbrace{\sum_{t=1}^T \Phi_t'(Q_t)R_t}_{\textrm{(II)}} + \underbrace{\sum_{t=1}^T \Phi_t'(Q_t) g_t(\bfx^\star)}_{\textrm{(III)}}.
\end{aligned}
\end{align}

This relation clarifies how $R_t$ works. In particular, to deduce \eqref{eq:Role of Reg:goal} from the above relation, the most problematic term is (I), whereas (III)---Lyapunov-weighted sum of comparator constraints---is nonpositive in expectation from the feasibility of $\bfx^\star$. We emphasize that, by adopting a carefully designed regularizer, (II) can offset this problematic term (I), resulting in \eqref{eq:Role of Reg:goal}. This observation motivates us to choose $R_t$ as
\[
    R_t \approx \frac{\sqrt{\sum_{\tau=1}^t\Phi_\tau'(Q_\tau)^2} - \sqrt{\sum_{\tau=1}^{t-1}\Phi_\tau'(Q_\tau)^2}}{\Phi_t'(Q_t)},
\]
which justifies the choice of $R_t$. Without the regularizer, it is nontrivial to bound (I), which highlights that our regularizer is essential for deriving small regret and stability of the dual process.

So far, we have demonstrated the mechanism of $R_t$ in the dual update. Although adopting $R_t$ is useful for achieving small regret and stability of the dual process, there is another caveat in the choice of $R_t$. That is, when $R_t$ is too large, the constraint violation cannot be bounded, based on the observation that $\Violation(T) = Q_{T+1} + \sum_{t=1}^T R_t$. Hence, there is a trade-off in the choice of $R_t$ in the sense that it is helpful to stabilize the dual process but, at the same time, it could harm the violation of the constraint. Nevertheless, under our choice of $R_t$, we prove that its cumulative sums are sublinear, e.g., \Cref{lem:R sum:stoc + exp + cvx}. Consequently, our choice of $R_t$ yields the desired nearly optimal guarantees without Slater's condition.

\subsection{Comparison with the Existing Dual Updates}
In this subsection, we compare our algorithm with the algorithms of \cite{sinha2024optimal} and \cite{yu2017online}. We first discuss the relation to \cite{sinha2024optimal}, which studied the adversarial constraint setting without Slater's condition. 
In particular, we focus on addressing the following question: why is $R_t$ unnecessary in the adversarial setting, whose dual update is $Q_t = Q_{t-1} + [g_t(\bfx_t)]_+$? To see this, we provide the analogue of \eqref{eq:Role of Reg:I II III}, which can be written as
\begin{align}\label{eq:Role of Reg:I II III:analouge}
\begin{aligned}
    &\Regret(T) + \Phi(Q_{T}) \lesssim  \bigO\left(\sqrt{T}\right)  + \underbrace{\sqrt{\sum_{t=1}^T\Phi'(Q_t)^2}}_{\mathrm{(I')}} + \underbrace{\sum_{t=1}^T \Phi'(Q_t) [g_t(\bfx^\star)]_+}_{\mathrm{(III')}}.
\end{aligned}
\end{align}

Since their dual update ensures that $\{Q_t\}_{t\geq 1}$ is an increasing sequence, choosing $\Phi(x) = e^{\gamma x}-1$ with $\gamma = 1/(2\sqrt{T})$ yields $\mathrm{(I')} \leq \Phi(Q_T)/2$. Hence, $\mathrm{(I')}$ can be absorbed into $\Phi(Q_{T})$ on the left-hand side, without relying on a regularizer. Moreover, $\mathrm{(III')}$ is $0$ under their choice of $\bfx^\star$, since $g_t(\bfx^\star) \leq 0$ for all $t$. These observations explain why the algorithm of \cite{sinha2024optimal} works without a regularizer.
However, this approach no longer holds in our setting. The main reason is that since we only have $\bbE[g_t(\bfx^\star)] \leq  0$ in the stochastic constraint setting, $\mathrm{(III')}$ does not coincide with $0$, even after taking  expectations. In summary, in the adversarial constraint setting, \eqref{eq:Role of Reg:goal} is attained without a regularizer, and this is possible due to the restrictive definition of $\bfx^\star$. This highlights both the challenge in the stochastic constraint setting and the importance of adopting $R_t$.

We conclude the section by comparing our approach with \cite{yu2017online}, which studies the same setting as ours but requires Slater's condition. This comparison explains why our approach does not require Slater's condition. To see this, under the quadratic Lyapunov function $\Phi(x) = x^2/2$ and the dual update $Q_{t+1} = [Q_t + g_t(\bfx_t) + \nabla g_t(\bfx_t)^\top(\bfx_{t+1} -\bfx_t)]_+$, their analysis shows that the conditional expected Lyapunov drift can be bounded as follows (see the proof of Lemma 7 in \cite{yu2017online}):
\begin{equation}
    \bbE[\Phi(Q_{t+t_0}) - \Phi(Q_{t})|\calF_{t-1}] \lesssim -\varepsilon t_0 \Phi'(Q_t) + C
\end{equation}
where $\varepsilon$ is a Slater constant, $t_0$ is a window parameter, $\calF_{t-1}$ is the $\sigma$-algebra generated by the history up to round $t-1$, and $C$ collects the remaining positive terms. This shows that the stability of their dual process stems from $-\varepsilon t_0 \Phi'(Q_t)$, as it is the only negative term on the right-hand side. Therefore, their analysis necessitates Slater's condition. In contrast, our analysis relies on the regularizer as in \eqref{eq:Role of Reg:drift}. This is the key reason why our algorithm is free from Slater's condition.

\section{Main Results}\label{sec:results}
In this section, we present the theoretical guarantees of our algorithm. In particular, we establish guarantees for both convex and strongly convex loss functions under stochastic and adversarial constraints. For the stochastic constraint setting, we further provide both expected and high-probability guarantees. 
An overview of the analysis is presented in \Cref{sec:overview of proofs}.


\subsection{Stochastic Constraints}\label{sec:results:stoc}
We first consider the stochastic constraint setting, where the loss functions may vary adversarially over time while the constraint functions are stochastic. Throughout this subsection, both the loss and constraint functions are assumed to be convex. To attain expected guarantees in this setting, recall that we choose the auxiliary function $\Psi$, the Lyapunov parameter $\gamma_t$ as in \eqref{eq:psi + gamma exp}, and the primal step size $\eta_t$ as in \eqref{eq:eta cvx}. We are now ready to present our first theoretical guarantee, where the corresponding analysis is provided in \Cref{sec:anal:stoc + exp + cvx}.
\begin{theorem}\label{thm:stoc + exp + cvx}
    Consider convex losses and stochastic constraints. For any $T\geq 1$, suppose that we run \Cref{alg:cumulative} for $T$ rounds with $\Psi, \gamma_t$ given by \eqref{eq:psi + gamma exp} and $\eta_t$ given by \eqref{eq:eta cvx}. Then, we have
    \begin{align*}
        \bbE[\Regret(T)] =  \bigO\left(\sqrt{T}\right), \quad \bbE[\Violation(T)] =  \bigO\left(\sqrt{T}\log T\right),
    \end{align*}
    where the expectation is taken with respect to the randomness of $\{g_t\}_{t\geq 1}$.
\end{theorem}
We note that this theorem demonstrates that our algorithm achieves nearly optimal expected guarantees in $T$, without relying on Slater's condition.

While \Cref{thm:stoc + exp + cvx} attains guarantees in expectation, it is often desirable to obtain realized guarantees that hold with high confidence. Thus, given a confidence parameter $\delta \in (0,1)$, we establish high-probability guarantees under the same setting as in \Cref{thm:stoc + exp + cvx}. To this end, we modify $\Psi$, $\gamma_t$ as in \eqref{eq:psi + gamma hp}, while retaining the same primal step size $\eta_t$ as in \eqref{eq:eta cvx}. Under these modified algorithmic parameters, the algorithm enjoys the following high-probability guarantees. The analysis of the following theorem is presented in \Cref{sec:anal:stoc + hp + cvx}.
\begin{theorem}\label{thm:stoc + hp + cvx}
Consider convex losses and stochastic constraints. For any $T\geq 1$ and $\delta \in (0,1)$, suppose that we run \Cref{alg:cumulative} for $T$ rounds with $\Psi, \gamma_t$ given by \eqref{eq:psi + gamma hp} and $\eta_t$ given by \eqref{eq:eta cvx}. Then, with probability at least $1-\delta$, we have
    \begin{align*}
        \Regret(T) =  \bigO\left(\sqrt{T} + \sqrt{\log(1/\delta)}\right), \quad \Violation(T) =  \bigO\left(\sqrt{T}\log(T/\delta) + \log^{3/2}(T/\delta)\right).
    \end{align*}
\end{theorem}

\subsection{Extension to Strongly Convex Losses}\label{sec:results:strcvx}
We next extend the preceding results to the setting with $\mu$-strongly convex losses, while the remaining assumptions are identical to those in \Cref{sec:results:stoc}. To obtain improved guarantees, it is sufficient to modify the primal step size $\eta_t$ as in \eqref{eq:eta strcvx}, which more properly exploits strong convexity (\Cref{lem:strcvx regret}). We then obtain the following expected and high-probability guarantees. 
\begin{theorem}\label{thm:stoc + exp + strcvx}
    Consider $\mu$-strongly convex losses for $\mu > 0$ and stochastic constraints. For any $T\geq 1$, suppose that we run \Cref{alg:cumulative} for $T$ rounds with $\Psi, \gamma_t$ given by \eqref{eq:psi + gamma exp} and $\eta_t$ given by \eqref{eq:eta strcvx}. Then, we have
    \begin{align*}
        \bbE[\Regret(T)] = \bigO\left(\log T\right), \quad \bbE[\Violation(T)] = \bigO\left(\sqrt{T}\log T\right),
    \end{align*}
    where the expectation is taken with respect to the randomness of $\{g_t\}_{t\geq 1}$.
\end{theorem}
\begin{theorem}\label{thm:stoc + hp + strcvx}
Consider $\mu$-strongly convex losses for $\mu >0$ and stochastic constraints. For all $T\geq 1$ and $\delta \in (0,1)$, suppose that we run \Cref{alg:cumulative} for $T$ rounds with $\Psi, \gamma_t$ given by \eqref{eq:psi + gamma hp} and $\eta_t$ given by \eqref{eq:eta strcvx}. Then, with probability at least $1-\delta$, we have
    \begin{align*}
        \Regret(T) =  \bigO\left(\log T + \sqrt{\log(1/\delta)}\right), \quad \Violation(T) =  \bigO\left(\sqrt{T}\log(T/\delta) + \log^{3/2}(T/\delta)\right).
    \end{align*}
\end{theorem}
Both theorems demonstrate that our algorithm can achieve improved expected and high-probability regret bounds that scale as $\log T$ when the losses are strongly convex. The analyses of Theorems~\ref{thm:stoc + exp + strcvx} and~\ref{thm:stoc + hp + strcvx} are presented in \Cref{sec:anal:stoc + strcvx}.

\subsection{Adversarial Constraints}\label{sec:results:adv}
Finally, we consider the adversarial constraint setting. Although \cite{sinha2024optimal} already established nearly optimal guarantees in this setting, the following results demonstrate that our framework is unified in the sense that it achieves nearly optimal guarantees under both stochastic and adversarial constraints.\footnote{Recently, this line of work has been extended to various settings~\citep{sinha2026beyond, sarkar2026projectionfree, sarkar2026improved, vaze2026sqrt, supantha2026universal}. In contrast, our focus is to demonstrate that the proposed algorithm applies to both stochastic and adversarial constraint settings. We leave extensions of our algorithm to these broader settings for future work.}

Although we may directly apply Algorithm~\ref{alg:cumulative} to the adversarial constraint setting, the resulting guarantees are on cumulative constraint violation, given by $\Violation(T)=\sum_{t=1}^T g_t(\bfx_t)$. Since this quantity allows cancellation between positive and negative violations across time, it is weaker than the notion of hard constraint violation, given by $\Violation_+(T)=\sum_{t=1}^T [g_t(\bfx_t)]_+$, which is often the performance metric adopted by adversarial-constraint algorithms. To obtain guarantees with respect to $\Violation_+(T)$, we introduce several modifications to Algorithm~\ref{alg:cumulative}.
\begin{algorithm}[t]
\caption{COCO for Hard Constraint Violation without Slater's condition}
\label{alg:hard}
\textbf{Input:} Lipschitz constant $L$; boundedness constant for constraints $G$; diameter $D$; confidence parameter $\delta \in (0,1)$; strong convexity constant $\mu$;
auxiliary function $\Psi:\bbR_+\to\bbR$; Lyapunov parameters $\{\gamma_t\}_{t\geq1}$.
\\
\textbf{Initialize:} Set $\bfx_1 \in \calX$, $Q_1 = 0$, and $\Phi_t(x) =e^{\gamma_t x} - 1$ for all $t\geq 1$.
\begin{algorithmic}
    \For{$t=1,2,\ldots$}
    \State Take $\bfx_t$ and observe $f_t,g_t$ 
    \State Set $\tilde g_t \leftarrow [g_t]_+$ and $\nabla \tilde g_t(\bfx_t) =\nabla g_t(\bfx_t)$ if $g_t(\bfx_t) > 0$, and $\bm 0$ otherwise
    \State Update the primal step size $\eta_t$
    \State Update $\bfx_{t+1} = \proj_\calX\left(\bfx_t - \eta_t(\nabla f_t(\bfx_t) + \Phi_t'(Q_t)\nabla \tilde g_t(\bfx_t))\right)$
    \State Update $R_{t}$:
    \begin{equation*}
        R_t = 12\gamma_t G^2 + \frac{\Psi\left(\sum_{\tau=1}^t\Phi_\tau'(Q_\tau)^2\right)-\Psi\left(\sum_{\tau=1}^{t-1}\Phi_\tau'(Q_\tau)^2\right)}{\Phi_t'(Q_t)}
    \end{equation*}
    \State Update $Q_{t+1} = Q_{t} + \tilde g_t(\bfx_t) - R_t$
    \EndFor
\end{algorithmic}
\end{algorithm}

The only difference is that \Cref{alg:hard} uses $\tilde g_t$ in place of $g_t$, where $\tilde g_t:\bfx\mapsto [g_t(\bfx)]_+$, and $\nabla \tilde g_t(\bfx_t)$ denotes a subgradient of $\tilde g_t$. Note that this modification yields hard constraint violation guarantees, since $\tilde g_t(\bfx_t)$ is accumulated in the dual variable update and $\sum_{t=1}^T \tilde g_t(\bfx_t)$ coincides with $\Violation_+(T)$. Hence, with \Cref{alg:hard}, we can obtain guarantees for $\Violation_+(T)$ by applying the same analysis strategy used for the stochastic constraint setting. This is possible because the comparator $\bfx^\star$ in the adversarial constraint setting is feasible in every round, i.e., $g_t(\bfx^\star)\leq 0$ for all $t\geq 1$, which is stronger than feasibility in expectation used in the stochastic constraint setting. Finally, we obtain the following results, whose analyses are deferred to Appendix~\ref{app:anal:adv}, as they follow from straightforward modifications of the previous analysis.

\begin{theorem}\label{thm:adv + exp + cvx}
Consider convex losses and adversarial constraints. For all $T\geq 1$, suppose that we run \Cref{alg:hard} for $T$ rounds with $\Psi, \gamma_t$ given by \eqref{eq:psi + gamma exp} and $\eta_t$ given by \eqref{eq:eta cvx}. Then, we have
    \begin{align*}
        \Regret(T) =  \bigO\left(\sqrt{T}\right), \quad \Violation_+(T) = \bigO\left(\sqrt{T}\log T\right).
    \end{align*}
\end{theorem}

\begin{theorem}\label{thm:adv + exp + strcvx}
Consider $\mu$-strongly convex losses for $\mu > 0$ and adversarial constraints. For all $T\geq 1$, suppose that we run \Cref{alg:hard} for $T$ rounds with $\Psi, \gamma_t$ given by \eqref{eq:psi + gamma exp} and $\eta_t$ given by \eqref{eq:eta strcvx}. Then, we have
    \begin{align*}
        \Regret(T) =  \bigO\left(\log T\right), \quad \Violation_+(T) = \bigO\left(\sqrt{T}\log T\right).
    \end{align*}
\end{theorem}

\subsection{Overview of Proofs}\label{sec:overview of proofs}
The analyses of all theorems follow the same template, while additional techniques are required for handling high-probability guarantees and strongly convex losses. First, we establish a drift bound for the time-varying exponential Lyapunov function. Second, we combine this drift bound with a surrogate regret guarantee with respect to $f_t + \Phi_t'(Q_t)g_t$. Third, the regularizer cancels the accumulated term involving $\sum_{t=1}^T \Phi_t'(Q_t)^2$. The remaining work is to control the Lyapunov-weighted comparator constraint term, which is nonpositive in expectation and controlled by a supermartingale concentration inequality with high probability, and then to bound the cumulative sum of regularizers $\sum_{t=1}^T R_t$.

\section{Analysis under Stochastic Constraints: Expected Bounds}\label{sec:anal:stoc + exp + cvx}

In this section, we present the analysis of \Cref{thm:stoc + exp + cvx}, which establishes expected guarantees under convex losses and stochastic constraints. Throughout, $\bbE$ denotes the expectation taken with respect to the randomness of $\{g_t\}_{t\geq 1}$, and $\calF_{t}$ denotes the $\sigma$-algebra generated by $\{f_\tau\}_{\tau\geq 1}$ and $\{g_\tau\}_{\tau=1}^t$.

\begin{lemma}\label{lem:drift:stoc + exp + cvx}
Suppose that we run \Cref{alg:cumulative} with $\gamma_t,\Psi$ as given by \eqref{eq:psi + gamma exp}. Then, for all $t\geq 1$, we have
    \begin{equation*} 
        \Phi_{t+1}(Q_{t+1}) - \Phi_t(Q_{t}) \leq \Phi_t'(Q_{t})\left(g_t(\bfx_t) - \frac{R_t}{2}\right) + \frac{\gamma_t - \gamma_{t+1}}{e\gamma_{t+1}}.
    \end{equation*}
\end{lemma}
\begin{proof}
    Note that the Lyapunov drift can be decomposed as
    \begin{align} \label{eq:lem1:(I)+(II):stoc}
        \Phi_{t+1}(Q_{t+1}) - \Phi_t(Q_t) = \underbrace{\Phi_t(Q_{t+1}) - \Phi_t(Q_t)}_{\textrm{(I)}} + \underbrace{\Phi_{t+1}(Q_{t+1}) - \Phi_{t}(Q_{t+1})}_{\textrm{(II)}}.
    \end{align}
    Now, we first bound (I). Recall that $\Phi_t(x) = e^{\gamma_t x} - 1$ and $\Phi_t'(x) = \gamma_t e^{\gamma_t x}$. It follows that
    \begin{align*}
        \Phi_{t}(Q_{t+1}) - \Phi_t(Q_t) 
        &= e^{\gamma_{t} Q_{t+1}} - e^{\gamma_t Q_t}\\
        &= e^{\gamma_t Q_t}(e^{\gamma_t (Q_{t+1}-Q_t)} - 1)\\
        &= e^{\gamma_t Q_t}(e^{\gamma_t (g_t(\bfx_t) - R_t)} - 1)\\
        &= \frac{1}{\gamma_t}\Phi_t'(Q_t)(e^{\gamma_t (g_t(\bfx_t) - R_t)} - 1).
    \end{align*}
    Note that $e^{x}\leq 1+ x + x^2$ for all $x\leq 1$, $\gamma_t (g_t(\bfx_t) - R_t) \leq \gamma_t G \leq 1$, and $\Phi_t'(Q_t)/\gamma_t > 0$. Then it follows that
    \begin{align}\label{eq:lem1:(I) decomp:stoc + exp + cvx}
        &\Phi_t(Q_{t+1}) - \Phi_t(Q_t)  \notag\\
        &\leq \frac{1}{\gamma_t} \Phi_t'(Q_t) \left(\gamma_t(g_t(\bfx_t)-R_t) + \gamma_t^2 (g_t(\bfx_t) - R_t)^2\right)\\
        &\leq  \Phi_t'(Q_t) \left(g_t(\bfx_t)-R_t + 3\gamma_tG^2 + 432\gamma_t^3G^4 + 3\gamma_t \left(\frac{\Psi\left(\sum_{\tau=1}^t\Phi_\tau'(Q_\tau)^2\right)-\Psi\left(\sum_{\tau=1}^{t-1}\Phi_\tau'(Q_\tau)^2\right)}{\Phi_t'(Q_t)}\right)^2\right)\notag
    \end{align}
    where the second inequality is due to the fact that $(x+y+z)^2 \leq 3(x^2+y^2+z^2)$ for any $x,y,z \in \bbR$. Since $\gamma_t \leq 1/(12G)$, we have
    \begin{align}\label{eq:lem1:(I) decomp 1:stoc + exp + cvx}
        3\gamma_t G^2 + 432\gamma_t^3G^4 \leq 3\gamma_t G^2 + 432\gamma_t G^4\cdot \frac{1}{144G^2} = 6\gamma_t G^2.
    \end{align}
    Moreover, we have
    \begin{align}\label{eq:lem1:(I) decomp 2:stoc + exp + cvx}
    \begin{aligned}
        &3\gamma_t \frac{\Psi\left(\sum_{\tau=1}^t\Phi_\tau'(Q_\tau)^2\right)-\Psi\left(\sum_{\tau=1}^{t-1}\Phi_\tau'(Q_\tau)^2\right)}{\Phi_t'(Q_t)} \\
        &= 12\gamma_tDL\frac{\sqrt{\sum_{\tau=1}^t\Phi_\tau'(Q_\tau)^2} - \sqrt{\sum_{\tau=1}^{t-1}\Phi_\tau'(Q_\tau)^2}}{\Phi_t'(Q_t)}\\
        &= \frac{12\gamma_t DL\Phi_t'(Q_t)}{\sqrt{\sum_{\tau=1}^t\Phi_\tau'(Q_\tau)^2} + \sqrt{\sum_{\tau=1}^{t-1}\Phi_\tau'(Q_\tau)^2}}\\
        &\leq 12\gamma_t DL\\
        &\leq \frac{1}{2}
    \end{aligned}
    \end{align}
    where the first inequality follows from $\Phi_t'(Q_t) \leq \sqrt{\sum_{\tau=1}^t\Phi_\tau'(Q_\tau)^2} + \sqrt{\sum_{\tau=1}^{t-1}\Phi_\tau'(Q_\tau)^2}$, and the second inequality is due to $\gamma_t\leq 1/(24DL)$. Note that $\Psi$ is an increasing function, implying that $\Psi\left(\sum_{\tau=1}^t\Phi_\tau'(Q_\tau)^2\right)-\Psi\left(\sum_{\tau=1}^{t-1}\Phi_\tau'(Q_\tau)^2\right)$ is nonnegative. Then, by applying \eqref{eq:lem1:(I) decomp 1:stoc + exp + cvx} and \eqref{eq:lem1:(I) decomp 2:stoc + exp + cvx} to \eqref{eq:lem1:(I) decomp:stoc + exp + cvx}, (I) is bounded as
    \begin{align*}
        \textrm{(I)}& = \Phi_{t}(Q_{t+1}) - \Phi_t(Q_t) \\
        &\leq \Phi_t'(Q_t) \left(g_t(\bfx_t)-R_t + 6\gamma_tG^2 +  \frac{\Psi\left(\sum_{\tau=1}^t\Phi_\tau'(Q_\tau)^2\right)-\Psi\left(\sum_{\tau=1}^{t-1}\Phi_\tau'(Q_\tau)^2\right)}{2\Phi_t'(Q_t)}\right) \\
        &= \Phi_t'(Q_t)\left(g_t(\bfx_t) - \frac{1}{2}R_t\right).
    \end{align*}
    Next, we bound (II). Since $\gamma_{t+1}\leq \gamma_t$, by \Cref{prop:Phi}, we have
    \[
    \textrm{(II)}=\Phi_{t+1}(Q_{t+1}) - \Phi_t(Q_{t+1}) \leq \frac{\gamma_t - \gamma_{t+1}}{e\gamma_{t+1}}.
    \]
    By applying bounds on (I) and (II) to \eqref{eq:lem1:(I)+(II):stoc}, we conclude the proof.
\end{proof}

From the right-hand side of the inequality in \Cref{lem:drift:stoc + exp + cvx}, we observe that the cost of employing time-varying Lyapunov functions is $(\gamma_t - \gamma_{t+1})/(e\gamma_{t+1})$. Indeed, the following lemma demonstrates that this cost is not significant by showing that its summation over $\tau=1,\ldots,t$ grows only logarithmically with $t$.
\begin{lemma}\label{lem:gamma sum exp}
    Consider $\gamma_t$ given by \eqref{eq:psi + gamma exp}. Then, we have
    \[
        \sum_{\tau=1}^t \frac{\gamma_\tau - \gamma_{\tau+1}}{\gamma_{\tau+1}} \leq \frac{1 + \log t}{2}.
    \]
\end{lemma}
\begin{proof}
    Recall that $1/\gamma_\tau = \max\{12G\sqrt{\tau}, 24DL, 1\}$. Then it follows that
    \begin{align*}
        \frac{\gamma_\tau - \gamma_{\tau+1}}{\gamma_{\tau+1}} 
        &= \frac{\frac{1}{\gamma_{\tau+1}} - \frac{1}{\gamma_\tau}}{\frac{1}{\gamma_\tau}} 
        \leq \frac{\sqrt{\tau+1} - \sqrt{\tau}}{\sqrt{\tau}}
        = \frac{1}{(\sqrt{\tau+1} + \sqrt{\tau})\sqrt{\tau}}
        \leq \frac{1}{2\tau}
    \end{align*}
    where the first inequality follows from $\max\{x,y,z\} - \max\{x',y,z\} \leq x-x'$ for any $x > x'$ and $1/\gamma_\tau \geq 12G\sqrt{\tau}$. Thus, we have
    \begin{align*}
        \sum_{\tau=1}^t \frac{\gamma_\tau - \gamma_{\tau+1}}{\gamma_{\tau+1}}  \leq \sum_{\tau=1}^t \frac{1}{2\tau} \leq \frac{1 + \log t}{2},
    \end{align*}
    as required.
\end{proof}

Next, we introduce \Cref{lem:adagrad}, which provides an AdaGrad-type guarantee for the surrogate loss $f_t+\Phi_t'(Q_t)g_t$.
\begin{lemma}\label{lem:adagrad}
        Consider the setting of \Cref{thm:stoc + exp + cvx}. Suppose that we run \Cref{alg:cumulative} with $\eta_t$ given by \eqref{eq:eta cvx}. Then, for all $t\geq 1$, we have
        \begin{align*}
            &\sum_{\tau=1}^t (f_\tau(\bfx_\tau) - f_\tau(\bfx^\star)) + \sum_{\tau=1}^t \Phi_\tau'(Q_{\tau})g_\tau(\bfx_\tau) \\
            &\leq \sum_{\tau=1}^t \Phi_\tau'(Q_\tau)g_\tau(\bfx^\star) + 2DL \sqrt{t}+ 2DL \sqrt{\sum_{\tau=1}^t\Phi_\tau'(Q_\tau)^2} + D.
        \end{align*}
\end{lemma}
\begin{proof}
    For simplicity, let $\hat f_\tau = f_\tau + \Phi_\tau'(Q_\tau)g_\tau$ for each $\tau \in [t]$. Note that since $\Phi_\tau'(Q_\tau) > 0$, $\hat f_\tau$ is convex. Since projection onto $\calX$ is a contraction mapping, we have for any $\bfx^\star \in \calX$,
    \begin{align*}
        \|\bfx_{\tau+1} - \bfx^\star\|^2 
        &\leq \| \bfx_\tau -\eta_\tau\nabla \hat f_\tau(\bfx_\tau) - \bfx^\star\|^2 \\
        &\leq \|\bfx_\tau - \bfx^\star\|^2 + \eta_\tau^2\|\nabla \hat f_\tau(\bfx_\tau)\|^2 - 2\eta_\tau\nabla \hat f_\tau(\bfx_\tau)^\top (\bfx_\tau - \bfx^\star).
    \end{align*}
    By convexity, we have $\hat f_\tau(\bfx_\tau) - \hat f_\tau(\bfx^\star)\leq \nabla \hat f_\tau(\bfx_\tau)^\top(\bfx_\tau - \bfx^\star)$. It follows that
    \begin{align*}
        \hat f_\tau(\bfx_\tau) - \hat f_\tau(\bfx^\star) \leq \frac{\|\bfx_\tau - \bfx^\star\|^2 - \|\bfx_{\tau+1} - \bfx^\star\|^2}{2\eta_\tau} + \frac{\eta_\tau \|\nabla \hat f_\tau(\bfx_\tau)\|^2}{2}.
    \end{align*}
    Summing over $\tau=1,\ldots, t$ yields
    \begin{align*}
        \sum_{\tau=1}^t(\hat f_\tau(\bfx_\tau) - \hat f_\tau(\bfx^\star)) \leq \sum_{\tau=1}^t\frac{\|\bfx_\tau - \bfx^\star\|^2 - \|\bfx_{\tau+1} - \bfx^\star\|^2}{2\eta_\tau} + \sum_{\tau=1}^t\frac{\eta_\tau \|\nabla \hat f_\tau(\bfx_\tau)\|^2}{2}.
    \end{align*}
    Note that 
    \begin{align*}       \sum_{\tau=1}^t\frac{\|\bfx_\tau - \bfx^\star\|^2 - \|\bfx_{\tau+1} - \bfx^\star\|^2}{2\eta_\tau} 
        &\leq \sum_{\tau=1}^{t-1}\frac{\|\bfx_{\tau+1}-\bfx^\star\|^2}{2}\left(\frac{1}{\eta_{\tau+1}} - \frac{1}{\eta_\tau}\right) + \frac{D^2}{2\eta_1} \\
        &\leq \frac{D^2}{2\eta_t}\\
        &= \frac{D}{\sqrt{2}}\sqrt{1 + \sum_{\tau=1}^t \|\nabla \hat f_\tau(\bfx_\tau)\|^2}\\
        &\leq\frac{D}{\sqrt{2}}\sqrt{\sum_{\tau=1}^t \|\nabla \hat f_\tau(\bfx_\tau)\|^2} + \frac{D}{\sqrt{2}}
    \end{align*}
    where the second inequality follows from $1/\eta_{\tau+1} - 1/\eta_\tau \geq 0$ and $\|\bfx_{\tau+1} - \bfx^\star\|^2 \leq D^2$, and the last inequality follows from $\sqrt{a+b}\leq \sqrt{a}+\sqrt{b}$ for any $a,b\geq 0$. Moreover,
    \begin{align*}
        \sum_{\tau=1}^t\frac{\eta_\tau \|\nabla \hat f_\tau(\bfx_\tau)\|^2}{2} 
        &= \frac{D}{2\sqrt{2}}\sum_{\tau=1}^t\frac{\|\nabla \hat f_\tau(\bfx_\tau)\|^2}{\sqrt{1 + \sum_{s=1}^\tau\|\nabla \hat f_s(\bfx_s)\|^2}}\\
        &\leq \frac{D}{2\sqrt{2}}\cdot 2\left(\sqrt{1 + \sum_{\tau=1}^t \|\nabla \hat f_\tau(\bfx_\tau)\|^2} - 1\right) \\
        &\leq \frac{D}{\sqrt{2}}\sqrt{\sum_{\tau=1}^t \|\nabla \hat f_\tau(\bfx_\tau)\|^2}
    \end{align*}
    where the second inequality follows from $\sum_{\tau=1}^t y_\tau/\sqrt{1+\sum_{s=1}^\tau y_s} \leq 2(\sqrt{1+\sum_{\tau=1}^ty_\tau}-1)$ for any nonnegative sequence $\{y_\tau\}_{\tau\geq 1}$, and the last inequality follows from $\sqrt{a+b}\leq \sqrt{a} + \sqrt{b}$ for any $a,b \geq0$. Since $1/\sqrt{2} \leq 1$, it follows that 
    \begin{align*}
        \sum_{\tau=1}^t(\hat f_\tau(\bfx_\tau) - \hat f_\tau(\bfx^\star)) \leq \sqrt{2}D\sqrt{\sum_{\tau=1}^t \|\nabla \hat f_\tau(\bfx_\tau)\|^2} + D.
    \end{align*}
    Recall that $\hat f_\tau = f_\tau + \Phi_\tau'(Q_\tau)g_\tau$ and $\|\nabla f_\tau(\bfx_\tau)\|^2, \|\nabla g_\tau(\bfx_\tau)\|^2 \leq L^2$. It follows that
    \begin{align*}
        &\sum_{\tau=1}^t (f_\tau(\bfx_\tau) - f_\tau(\bfx^\star)) + \sum_{\tau=1}^t \Phi_\tau'(Q_\tau)(g_\tau(\bfx_\tau) - g_\tau(\bfx^\star)) \\
        &\leq 2DL\sqrt{t} + 2DL\sqrt{\sum_{\tau=1}^t \Phi_\tau'(Q_\tau)^2} + D.
    \end{align*}
    By rearranging the above inequality, we conclude the proof.
\end{proof}

We now combine Lemmas~\ref{lem:drift:stoc + exp + cvx} and~\ref{lem:adagrad}. This yields the following lemma, which establishes both a small regret bound and the stability of the dual process in expectation.
\begin{lemma}\label{lem:Regret+Phi:stoc + exp + cvx}
    Consider the setting of \Cref{thm:stoc + exp + cvx}. Suppose that we run \Cref{alg:cumulative} with $\gamma_t,\Psi$ given by \eqref{eq:psi + gamma exp} and $\eta_t$ given by \eqref{eq:eta cvx}. Then, for all $t\geq 1$, we have
    \begin{align*}
        \bbE\left[\sum_{\tau=1}^t (f_\tau(\bfx_\tau) - f_\tau(\bfx^\star))\right] + \bbE\left[\Phi_{t+1}(Q_{t+1})\right] \leq 2DL\sqrt{t} +  \log t + 1 + D.
    \end{align*}
\end{lemma}
\begin{proof}
    By \Cref{lem:adagrad}, we have
    \begin{align*}
    \begin{aligned}
        &\sum_{\tau=1}^t (f_\tau(\bfx_\tau)-f_\tau(\bfx^\star)) + \sum_{\tau=1}^t \Phi_\tau'(Q_\tau)g_\tau(\bfx_\tau) \\
        &\leq\sum_{\tau=1}^t\Phi_\tau'(Q_\tau)g_\tau(\bfx^\star) + 2DL\sqrt{t} + 2DL\sqrt{\sum_{\tau=1}^t \Phi_\tau'(Q_\tau)^2}
        + D.
    \end{aligned}
    \end{align*}
    By summing the inequality in \Cref{lem:drift:stoc + exp + cvx} over $\tau=1, \ldots,t$, since $\Phi_1(Q_1) = 0$, we have
    \begin{align*}
        \Phi_{t+1}(Q_{t+1}) - \Phi_1(Q_1) \leq \sum_{\tau=1}^t \Phi_\tau'(Q_\tau)\left(g_\tau(\bfx_\tau) - \frac{R_\tau}{2}\right) + \sum_{\tau=1}^t\frac{\gamma_\tau - \gamma_{\tau+1}}{e\gamma_{\tau+1}}.
    \end{align*}
    Since $\Phi_1(Q_1) = 0$, it follows that
    \begin{align*}
        &\sum_{\tau=1}^t (f_\tau(\bfx_\tau)-f_\tau(\bfx^\star)) + \Phi_{t+1}(Q_{t+1}) \\
        &\leq\sum_{\tau=1}^t\Phi_\tau'(Q_\tau)g_\tau(\bfx^\star) + 2DL\sqrt{t} + 2DL\sqrt{\sum_{\tau=1}^t \Phi_\tau'(Q_\tau)^2}
        + D + \sum_{\tau=1}^t\frac{\gamma_\tau - \gamma_{\tau+1}}{e\gamma_{\tau+1}} - \frac{1}{2}\sum_{\tau=1}^t \Phi_\tau'(Q_\tau)R_\tau.
    \end{align*}
    Note that 
    \begin{align*}
        \frac{1}{2}\sum_{\tau=1}^t \Phi_\tau'(Q_\tau)R_\tau 
        &\geq \frac{1}{2}\sum_{\tau=1}^t \Phi_\tau'(Q_\tau)\cdot\frac{\Psi(\sum_{s=1}^\tau \Phi_s'(Q_s)^2)-\Psi(\sum_{s=1}^{\tau-1}\Phi_s'(Q_s)^2)}{\Phi_\tau'(Q_\tau)} \\
        &= \frac{1}{2}\Psi\left(\sum_{\tau=1}^t \Phi_\tau'(Q_\tau)^2\right) - \frac{1}{2}\Psi\left(0\right)\\
        &=2DL\sqrt{\sum_{\tau=1}^t \Phi_\tau'(Q_\tau)^2}.
    \end{align*}
    Thus, we deduce that
    \begin{align*}
        \sum_{\tau=1}^t (f_\tau(\bfx_\tau)-f_\tau(\bfx^\star)) + \Phi_{t+1}(Q_{t+1}) 
        &\leq \sum_{\tau=1}^t\Phi_\tau'(Q_\tau)g_\tau(\bfx^\star)+2DL\sqrt{t}
         + \sum_{\tau=1}^t\frac{\gamma_\tau - \gamma_{\tau+1}}{e\gamma_{\tau+1}} + D\\
        &\leq \sum_{\tau=1}^t\Phi_\tau'(Q_\tau)g_\tau(\bfx^\star)+2DL\sqrt{t}
         + \log t + 1 + D.
    \end{align*}
    where the second inequality follows from \Cref{lem:gamma sum exp} and $1/(2e) \leq 1$. Now, taking expectations on both sides yields
    \begin{align*}
        &\bbE\left[\sum_{\tau=1}^t (f_\tau(\bfx_\tau)-f_\tau(\bfx^\star))\right] + \bbE\left[\Phi_{t+1}(Q_{t+1})\right]\\
        &\leq \sum_{\tau=1}^t \bbE\left[\Phi_\tau'(Q_\tau)g_\tau(\bfx^\star)\right]+2DL\sqrt{t}
         +  \log t + 1 + D\\
         &= \sum_{\tau=1}^t \bbE\left[\Phi_\tau'(Q_\tau)\bbE[g_\tau(\bfx^\star)|\calF_{\tau-1}]\right]+2DL\sqrt{t}
         + \log t + 1 + D \\
         &\leq 2DL\sqrt{t}
         +  \log t + 1 + D
    \end{align*}
    where the equality follows from the tower rule and $\Phi_\tau'(Q_\tau)$ is $\calF_{\tau-1}$-measurable, and the second inequality follows from the feasibility assumption that $\bbE[g_\tau(\bfx^\star)] \leq 0$ for fixed $\bfx^\star$ and that $\{g_\tau\}_{\tau\geq 1}$ are i.i.d. samples. This concludes the proof.
\end{proof}

Next, we introduce \Cref{lem:R sum:stoc + exp + cvx}, which shows that the cumulative sum of $R_\tau$ up to round $t$ scales as $\sqrt{t \log t}$ in expectation.
\begin{lemma}\label{lem:R sum:stoc + exp + cvx}
    Consider the setting of \Cref{thm:stoc + exp + cvx}. Suppose that we run \Cref{alg:cumulative} with $\gamma_t,\Psi$ given by \eqref{eq:psi + gamma exp}, and $\eta_t$ given by \eqref{eq:eta cvx}. Then, for all $t\geq 1$, we have
    \begin{align*}
        \bbE\left[\sum_{\tau=1}^t R_\tau\right] =  \bigO\left(\sqrt{t \log t}\right).
    \end{align*}
\end{lemma}
\begin{proof}
    We first bound $\bbE[\Phi_{t+1}'(Q_{t+1})]$, which is frequently used throughout the proof. By Lipschitzness, we have $f_\tau(\bfx_\tau) - f_\tau(\bfx^\star) \geq -LD$. Then it leads to $\sum_{\tau=1}^t (f_\tau(\bfx_\tau) - f_\tau(\bfx^\star))\geq -LDt$ almost surely. By \Cref{lem:Regret+Phi:stoc + exp + cvx}, we have
    \begin{align*}
        \bbE[\Phi_{t+1}(Q_{t+1})] 
        &\leq 3DLt +  \log t + 1 + D.
    \end{align*}
    This leads to
    \begin{align}\label{eq:Phi' bound:stoc + exp + cvx}
    \begin{aligned}
        \bbE[\Phi_{t+1}'(Q_{t+1})]
        &\leq \gamma_{t+1}\left(3DLt + \log t +  D + 2\right)\leq 3DLt + \log t +  D + 2.
    \end{aligned}
    \end{align}
    Now, we bound $\sum_{\tau=1}^t R_\tau$. Since $\gamma_\tau \leq 1/(12G\sqrt{\tau})$, we have
    \begin{align*}
        \bbE\left[\sum_{\tau=1}^t R_\tau\right] 
        \leq \sum_{\tau=1}^t\frac{G}{\sqrt{\tau}} + \sum_{\tau=1}^t
        \bbE\left[\frac{\Psi\left(\sum_{s=1}^\tau\Phi_s'(Q_s)^2\right)-\Psi\left(\sum_{s=1}^{\tau-1}\Phi_s'(Q_s)^2\right)}{\Phi_\tau'(Q_\tau)}\right].
    \end{align*}
    It is clear that $\sum_{\tau=1}^t G/\sqrt{\tau} \leq 2G\sqrt{t}$. Moreover, it follows that
    \begin{align*}
        &\sum_{\tau=1}^t\frac{\Psi\left(\sum_{s=1}^\tau\Phi_s'(Q_s)^2\right)-\Psi\left(\sum_{s=1}^{\tau-1}\Phi_s'(Q_s)^2\right)}{\Phi_\tau'(Q_\tau)} \\
        &= 4DL\sum_{\tau=1}^t \frac{\sqrt{ \sum_{s=1}^\tau\Phi_s'(Q_s)^2}-\sqrt{ \sum_{s=1}^{\tau-1}\Phi_s'(Q_s)^2}}{\Phi_\tau'(Q_\tau)}\\
        &= 4DL\sum_{\tau=1}^t\frac{\Phi_\tau'(Q_\tau)}{\sqrt{ \sum_{s=1}^\tau\Phi_s'(Q_s)^2}+\sqrt{ \sum_{s=1}^{\tau-1}\Phi_s'(Q_s)^2}}\\
        &\leq 4DL\sum_{\tau=1}^t\frac{\Phi_\tau'(Q_\tau)}{\sqrt{ \sum_{s=1}^\tau\Phi_s'(Q_s)^2}} \\
        &\leq 4DL\sqrt{t}\sqrt{\sum_{\tau=1}^t\frac{\Phi_\tau'(Q_\tau)^2}{\sum_{s=1}^\tau \Phi_s'(Q_s)^2}}\\
        &= 4DL\sqrt{t}\sqrt{\sum_{\tau=2}^t\frac{\sum_{s=1}^\tau\Phi_s'(Q_s)^2 - \sum_{s=1}^{\tau-1}\Phi_s'(Q_s)^2}{\sum_{s=1}^\tau \Phi_s'(Q_s)^2} + 1} \\
        &\leq 4DL\sqrt{t}\sqrt{\log\frac{\sum_{\tau=1}^t \Phi_\tau'(Q_\tau)^2}{\Phi_1'(Q_1)^2} +1}\\
        &\leq 4DL\sqrt{t}\sqrt{2\log\frac{\sum_{\tau=1}^t \Phi_\tau'(Q_\tau)}{\Phi_1'(Q_1)} +1}.
    \end{align*}
    where the second inequality follows from the Cauchy-Schwarz inequality, the third inequality is due to \Cref{lem:integral 1/x}, and the last inequality follows from $\sum_{\tau=1}^t a_\tau^2 \leq (\sum_{\tau=1}^t a_\tau)^2$ for any nonnegative sequence $\{a_\tau\}_{\tau=1}^t$. Now, by taking $\bbE$ on both sides, since $\sqrt{x}$ and $\log x$ are concave functions, it follows that
    \begin{align*}
        &\sum_{\tau=1}^t\bbE\left[\frac{\Psi\left(\sum_{s=1}^\tau\Phi_s'(Q_s)^2\right)-\Psi\left(\sum_{s=1}^{\tau-1}\Phi_s'(Q_s)^2\right)}{\Phi_\tau'(Q_\tau)}\right] \\
        &\leq \bbE\left[4DL\sqrt{t}\sqrt{2\log\frac{\sum_{\tau=1}^t \Phi_\tau'(Q_\tau)}{\Phi_1'(Q_1)} +1}\right]\\
        &\leq 4DL\sqrt{t}\sqrt{2\log\frac{\sum_{\tau=1}^t\bbE[\Phi_\tau'(Q_\tau)]}{\Phi_1'(Q_1)} +1} \\
        &\leq 4DL\sqrt{t}\sqrt{2\log\frac{t(3DLt + \log t +  D + 2)}{\gamma_1} + 1}
    \end{align*}
    where the second inequality follows from Jensen's inequality, and the last inequality follows from \eqref{eq:Phi' bound:stoc + exp + cvx}. Consequently, we deduce that
    \begin{align*}
        \bbE\left[\sum_{\tau=1}^t R_\tau\right] \leq 2G\sqrt{t} + 4DL\sqrt{t}\sqrt{2\log\frac{t(3DLt + \log t +  D + 2)}{\gamma_1}+1}.
    \end{align*}
    It can be written as $\bbE\left[\sum_{\tau=1}^t R_\tau\right] =  \bigO\left(\sqrt{t \log t}\right)$.
\end{proof}

\subsection{Proof of Theorem~\ref{thm:stoc + exp + cvx}}
Now, we are ready to prove \Cref{thm:stoc + exp + cvx}.
\begin{proof}
    By \Cref{lem:Regret+Phi:stoc + exp + cvx}, since $\Phi_{t+1}(Q_{t+1}) \geq -1$, we have
    \begin{align*}
        \bbE[\Regret(t)]=\bbE\left[\sum_{\tau=1}^t (f_\tau(\bfx_\tau) - f_\tau(\bfx^\star))\right]\leq 2DL\sqrt{t} +  \log t + 2 + D.
    \end{align*}
    Next, to bound $\bbE[\Violation(t)]$, we observe that $\bbE\left[\sum_{\tau=1}^t g_\tau(\bfx_\tau)\right] = \bbE[Q_{t+1}] + \sum_{\tau=1}^t \bbE[R_\tau]$. Thus, we first bound $\bbE[Q_{t+1}]$. As done in the proof of \Cref{lem:R sum:stoc + exp + cvx}, we have
    \begin{align*}
    \begin{aligned} 
        \bbE[\Phi_{t+1}(Q_{t+1})]\leq 3DLt +  \log t + 1 + D.
    \end{aligned}
    \end{align*}
    By Jensen's inequality, we have $e^{\gamma_{t+1}\bbE[Q_{t+1}]}-1 \leq \bbE[e^{\gamma_{t+1}Q_{t+1}}-1] = \bbE[\Phi_{t+1}(Q_{t+1})]$. Thus, we have
    \begin{align*}
        \bbE[Q_{t+1}] 
        &\leq \frac{1}{\gamma_{t+1}}\log\left(3DLt +  \log t + 2 + D\right)\\
        &\leq \left(12G\sqrt{t+1} + 24DL + 1\right)\log\left(3DLt +  \log t + 2 + D\right)
    \end{align*}
    It can be written as
    \[
        \bbE[Q_{t+1}] =  \bigO\left(\sqrt{t}\log t\right).
    \]
    Moreover, by \Cref{lem:R sum:stoc + exp + cvx}, we have $\bbE\left[\sum_{\tau=1}^t R_\tau\right] =  \bigO(\sqrt{t \log t})$. 
    Consequently, we have
    \[
        \bbE[\Regret(t)] =  \bigO\left(\sqrt{t}\right), \quad \bbE[\Violation(t)] =  \bigO\left(\sqrt{t}\log t\right),
    \]
    as required.
\end{proof}

\section{Analysis under Stochastic Constraints: High-Probability Bounds}\label{sec:anal:stoc + hp + cvx}
In this section, we present the analysis of \Cref{thm:stoc + hp + cvx}, which establishes high-probability guarantees under convex losses and stochastic constraints. The key challenge is to obtain a high-probability bound on
\[
\sum_{\tau=1}^t \Phi_\tau'(Q_\tau)g_\tau(\bfx^\star),
\]
which is required to derive a high-probability analogue of \Cref{lem:Regret+Phi:stoc + exp + cvx}. In contrast, in the expected bound setting, the above summation is nonpositive in expectation as a consequence of feasibility, i.e., $\bbE[g_t(\bfx^\star)]\leq 0$. Therefore, we must derive a high-probability upper bound on this summation under the sole assumption that $\bbE[g_t(\bfx^\star)]\leq 0$. To this end,  we introduce the following lemma, which provides an anytime concentration inequality for supermartingales.
\begin{lemma}
    \label{lem:supmg ci}
    Let $\{X_t\}_{t\geq 1}$ be a random process adapted to a filtration $\{\calF_t\}_{t\geq 0}$, satisfying $\bbE[X_t|\calF_{t-1}] \leq 0$ for all $t\geq 1$. Let $\{w_t\}_{t\geq1}$ be a random process such that $w_t$ is positive and $\calF_{t-1}$-measurable for all $t\geq 1$. For some constant $G$, suppose that $|X_t| \leq G$ almost surely for all $t\geq 1$. Define $\{S_t\}_{t\geq 1}$ such that 
    \[
        S_t = \sum_{\tau=1}^t w_\tau X_\tau, \quad \forall t \geq 1.
    \]
    Let $\delta \in (0,1)$. With probability at least $1-\delta$, for all $t \geq 1$,
    \[
        S_t \leq 2G\sqrt{\left(1 + \sum_{\tau=1}^t w_\tau^2\right) \log \frac{(\pi^2/6)(1+\log_2(1 + \sum_{\tau=1}^t w_\tau^2))^2}{\delta}}.
    \]
\end{lemma}
\begin{proof}
    Since $|X_t| \leq G$, we have $|w_tX_t| \leq Gw_t$ almost surely. Fix $\lambda > 0$. By \Cref{lem:hoeffding lemma}, we have for each $t\geq 1$,
    \begin{align*}
        \bbE\left[\exp\left({\lambda w_t X_t}\right) | \calF_{t-1}\right] \leq \exp\left(\frac{\lambda^2 G^2 w_t^2}{2}\right)
    \end{align*}
    Since $w_t$ is $\calF_{t-1}$-measurable, it can be rewritten as
    \begin{align}\label{eq:lem:supmt ci 1}
        \bbE\left[\exp\left({\lambda w_t X_t}- \frac{\lambda^2G^2w_t^2}{2}\right) | \calF_{t-1}\right] \leq 1.
    \end{align}
    Define $\{M_t\}_{t\geq0}$ such that $M_0=1$ and 
    \begin{align*}
        M_t = M_{t-1} \exp\left({\lambda w_t X_t}- \frac{\lambda^2G^2w_t^2}{2}\right), \quad \forall  t \geq 1.
    \end{align*}
    Note that $\{M_t\}_{t\geq 0}$ is adapted to $\{\calF_t\}_{t\geq 0}$. Moreover, it is a supermartingale, since
    \[
        \bbE[M_t|\calF_{t-1}] \leq M_{t-1}\bbE\left[\exp\left({\lambda w_t X_t}- \frac{\lambda^2G^2w_t^2}{2}\right) | \calF_{t-1}\right] \leq M_{t-1}.
    \]
    By \Cref{lem:maximal ineq}, it follows that for any $\epsilon > 0$,
    \begin{align*}
        \bbP\left(\sup_{t\geq 0} M_t \geq \epsilon\right) \leq \frac{\bbE[M_0]}{\epsilon} = \frac{1}{\epsilon}.
    \end{align*}
    Note that $\bbP(\sup_{t\geq 1}M_t \geq \epsilon) \leq \bbP(\sup_{t\geq 0}M_t \geq \epsilon)$. Moreover, $M_t = \exp\left(\lambda S_t - (\lambda^2 G^2/2) \sum_{\tau=1}^t w_\tau^2\right)$ for every $t\geq 1$. By taking $\epsilon = 1/\delta$, it follows that
    \begin{align*}
        &\quad \bbP\left(\exists t\geq 1:\; \exp\left(\lambda S_t - \frac{\lambda^2G^2}{2} \sum_{\tau=1}^t w_\tau^2\right)\geq \frac{1}{\delta} \right) \leq \delta \\
        \Leftrightarrow &\quad \bbP\left(\exists t\geq 1:\;   S_t \geq \frac{\lambda G^2}{2} \sum_{\tau=1}^t w_\tau^2 + \frac{1}{\lambda}\log\frac{1}{\delta} \right) \leq \delta.
    \end{align*}
    Note that this implies that for any $\delta \in (0,1)$,
    \[
        \bbP\left(\exists t\geq 1:\;   S_t \geq \frac{\lambda G^2}{2} \left(1+\sum_{\tau=1}^t w_\tau^2\right) + \frac{1}{\lambda}\log\frac{1}{\delta} \right) \leq \delta.
    \]
    To conclude the proof, we apply the peeling argument. 
    In particular, let the index set $\calI_k$ be defined as for $k \geq 0$,
    \[
        \calI_k = \left\{t\geq 1:\; 2^k\leq 1 + \sum_{\tau=1}^t w_\tau^2 < 2^{k+1}\right\}.
    \]
    Furthermore, for each $t \in \calI_k$, we have
    \begin{align*}
        \frac{\lambda G^2}{2} \left(1+\sum_{\tau=1}^t w_\tau^2\right) + \frac{1}{\lambda}\log\frac{1}{\delta} 
        &< {\lambda G^2}2^{k} + \frac{1}{\lambda}\log\frac{1}{\delta} \\
        &= 2G\sqrt{2^k \log \frac{1}{\delta}} \\
        &\leq 2G\sqrt{\left(1 + \sum_{\tau=1}^t w_\tau^2\right) \log \frac{1}{\delta}}
    \end{align*}
    where the equality follows from taking $\lambda = \sqrt{\log(1/\delta)/G^22^k}$, and the inequalities follow from the definition of $\calI_k$. Then it follows that for any $\delta \in (0,1)$,
    \begin{align*}
        \bbP\left(\exists t \in \calI_k:\;   S_t \geq 2G\sqrt{\left(1 + \sum_{\tau=1}^t w_\tau^2\right) \log \frac{1}{\delta}} \right) \leq \delta.
    \end{align*}
    Moreover, given $\delta \in (0,1)$, define $\delta_k = \frac{\delta}{(\pi^2/6)(k+1)^2}$, hence we have $\sum_{k=0}^\infty \delta_k = \delta$, as $\sum_{k=0}^\infty \frac{1}{(k+1)^2} = \frac{\pi^2}{6}$. By taking $\delta \leftarrow \delta_k$, it follows that
    \begin{align*}
        &\bbP\left(\exists t \in \calI_k:\;   S_t \geq 2G\sqrt{\left(1 + \sum_{\tau=1}^t w_\tau^2\right) \log \frac{1}{\delta_k}} \right) \leq \delta_k\\
        \Leftrightarrow \quad&
        \bbP\left(\exists t \in \calI_k:\;   S_t \geq 2G\sqrt{\left(1 + \sum_{\tau=1}^t w_\tau^2\right) \log \frac{(\pi^2/6)(k+1)^2}{\delta}} \right) \leq \delta_k.
    \end{align*}
    For each $t \in \calI_k$, we know that $k \leq \log_2(1 + \sum_{\tau=1}^t w_\tau^2)$. This implies that
    \begin{align*}
        \bbP\left(\exists t \in \calI_k:\;   S_t \geq 2G\sqrt{\left(1 + \sum_{\tau=1}^t w_\tau^2\right) \log \frac{(\pi^2/6)(1+\log_2(1 + \sum_{\tau=1}^t w_\tau^2))^2}{\delta}} \right) \leq \delta_k.
    \end{align*}
    Since $\{t \geq 1\} = \bigcup_{k=0}^\infty \calI_k$, by taking union bound over $k = 0,1,2,\ldots$, we have
    \begin{align*}
        \bbP\left(\exists t \geq1:\;   S_t \geq 2G\sqrt{\left(1 + \sum_{\tau=1}^t w_\tau^2\right) \log \frac{(\pi^2/6)(1+\log_2(1 + \sum_{\tau=1}^t w_\tau^2))^2}{\delta}} \right) \leq \sum_{k=0}^\infty \delta_k = \delta.
    \end{align*}
    Thus, with probability at least $1-\delta$, for all $t \geq 1$, we have
    \[
        S_t \leq 2G\sqrt{\left(1 + \sum_{\tau=1}^t w_\tau^2\right) \log \frac{(\pi^2/6)(1+\log_2(1 + \sum_{\tau=1}^t w_\tau^2))^2}{\delta}}.
    \]
\end{proof}

Note that \Cref{lem:supmg ci} suggests that
\[
\sum_{\tau=1}^t \Phi_\tau'(Q_\tau)g_\tau(\bfx^\star)
\leq
2G\sqrt{\left(1 + \sum_{\tau=1}^t\Phi_\tau'(Q_\tau)^2\right)\log\frac{(\pi^2/6)(1+\log_2(1+\sum_{\tau=1}^t\Phi_\tau'(Q_\tau)^2))^2}{\delta}}.
\]
This observation motivates us to modify the definition of $\Psi$ compared to the expected bound setting. Specifically, since
\[
\bbE\left[\sum_{\tau=1}^t \Phi_\tau'(Q_\tau)g_\tau(\bfx^\star)\right]\leq 0
\]
in the expected bound setting, the only term that needs to be canceled is
\[
2DL\sqrt{\sum_{\tau=1}^t \Phi_\tau'(Q_\tau)^2},
\]
which leads to the choice $\Psi(x)=4DL\sqrt{x}$. In contrast, under the high-probability setting, the term that must be canceled becomes
\[
2DL\sqrt{\sum_{\tau=1}^t \Phi_\tau'(Q_\tau)^2}
+
2G\sqrt{\left(1 + \sum_{\tau=1}^t\Phi_\tau'(Q_\tau)^2\right)\log\frac{(\pi^2/6)(1+\log_2(1+\sum_{\tau=1}^t\Phi_\tau'(Q_\tau)^2))^2}{\delta}}.
\]
This leads to the modified choice of $\Psi$ in \eqref{eq:psi + gamma hp}. While the overall structure of the algorithm remains unchanged, the high-probability setting requires a more sophisticated choice of both $\Psi$ and $\gamma_t$. We now proceed to establish high-probability analogues of the previous lemmas under these updated parameter choices.

\begin{lemma}\label{lem:drift:stoc + hp + cvx}
Suppose that we run \Cref{alg:cumulative} with $\gamma_t,\Psi$ given by \eqref{eq:psi + gamma hp}. Then, for all $t\geq 1$, we have
    \begin{equation}\label{eq:time-varying drift}
        \Phi_{t+1}(Q_{t+1}) - \Phi_t(Q_{t}) \leq \Phi_t'(Q_{t})\left(g_t(\bfx_t) - \frac{R_t}{2}\right) + \frac{\gamma_t - \gamma_{t+1}}{e\gamma_{t+1}}.
    \end{equation}
\end{lemma}
\begin{proof}
    Note that the left-hand side of \eqref{eq:time-varying drift} can be decomposed as
    \begin{align}\label{eq:lem1:(I)+(II):hp}
        \Phi_{t+1}(Q_{t+1}) - \Phi_t(Q_t) = \underbrace{\Phi_t(Q_{t+1}) - \Phi_t(Q_t)}_{\textrm{(I)}} + \underbrace{\Phi_{t+1}(Q_{t+1}) - \Phi_{t}(Q_{t+1})}_{\textrm{(II)}}.
    \end{align}
    Now, we first bound (I). Recall that $\Phi_t(x) = e^{\gamma_t x} - 1$ and $\Phi_t'(x) = \gamma_t e^{\gamma_t x}$. It follows that
    \begin{align*}
        \Phi_{t}(Q_{t+1}) - \Phi_t(Q_t) 
        &= e^{\gamma_{t} Q_{t+1}} - e^{\gamma_t Q_t}\\
        &= e^{\gamma_t Q_t}(e^{\gamma_t (Q_{t+1}-Q_t)} - 1)\\
        &= e^{\gamma_t Q_t}(e^{\gamma_t (g_t(\bfx_t) - R_t)} - 1)\\
        &= \frac{1}{\gamma_t}\Phi_t'(Q_t)(e^{\gamma_t (g_t(\bfx_t) - R_t)} - 1).
    \end{align*}
    Note that $e^{x}\leq 1+ x + x^2$ for all $x\leq 1$, $\gamma_t (g_t(\bfx_t) - R_t) \leq \gamma_t G \leq 1$, and $\Phi_t'(Q_t)/\gamma_t > 0$. Then it follows that
    \begin{align}\label{eq:lem1:(I) decomp}
        &\Phi_t(Q_{t+1}) - \Phi_t(Q_t)  \notag\\
        &\leq \frac{1}{\gamma_t} \Phi_t'(Q_t) \left(\gamma_t(g_t(\bfx_t)-R_t) + \gamma_t^2 (g_t(\bfx_t) - R_t)^2\right)\\
        &\leq  \Phi_t'(Q_t) \left(g_t(\bfx_t)-R_t + 3\gamma_tG^2 + 432\gamma_t^3G^4 + 3\gamma_t \left(\frac{\Psi\left(\sum_{\tau=1}^t\Phi_\tau'(Q_\tau)^2\right)-\Psi\left(\sum_{\tau=1}^{t-1}\Phi_\tau'(Q_\tau)^2\right)}{\Phi_t'(Q_t)}\right)^2\right)\notag
    \end{align}
    where the second inequality is due to the Cauchy-Schwarz inequality. Since $\gamma_t \leq 1/(12G)$, we have
    \begin{align}\label{eq:lem1:(I) decomp 1}
        3\gamma_t G^2 + 432\gamma_t^3G^4 \leq 3\gamma_t G^2 + 432\gamma_t G^4\cdot \frac{1}{144G^2} = 6\gamma_t G^2.
    \end{align}
    Moreover, we have
    \begin{align}\label{eq:lem1:(I) decomp 2}
    \begin{aligned}
        &3\gamma_t \frac{\Psi\left(\sum_{\tau=1}^t\Phi_\tau'(Q_\tau)^2\right)-\Psi\left(\sum_{\tau=1}^{t-1}\Phi_\tau'(Q_\tau)^2\right)}{\Phi_t'(Q_t)} \\
        &= \underbrace{12\gamma_tDL\frac{\sqrt{\sum_{\tau=1}^t\Phi_\tau'(Q_\tau)^2} - \sqrt{\sum_{\tau=1}^{t-1}\Phi_\tau'(Q_\tau)^2}}{\Phi_t'(Q_t)}}_{\textrm{(I-a)}} \\
        &\quad+ \underbrace{
        12\gamma_tG
        \frac{\sqrt{\psi\left(\sum_{\tau=1}^t\Phi_\tau'(Q_\tau)^2\right)} - \sqrt{\psi\left(\sum_{\tau=1}^{t-1}\Phi_\tau'(Q_\tau)^2\right)}}{\Phi_t'(Q_t)}
        }_{\textrm{(I-b)}}
    \end{aligned}
    \end{align}
    where $\psi(x) = (1+x)\log\frac{(\pi^2/6)(1+\log_2(1+x))^2}{\delta}$. For (I-a), we have
    \begin{align}\label{eq:lem1:(I-a)}
    \begin{aligned}
        \textrm{(I-a)} 
        &=  \frac{12\gamma_t DL\Phi_t'(Q_t)}{\sqrt{\sum_{\tau=1}^t\Phi_\tau'(Q_\tau)^2} + \sqrt{\sum_{\tau=1}^{t-1}\Phi_\tau'(Q_\tau)^2}}\\
        &\leq 12\gamma_t DL
    \end{aligned}
    \end{align}
    where the inequality follows from $\Phi_t'(Q_t) \leq \sqrt{\sum_{\tau=1}^t\Phi_\tau'(Q_\tau)^2} + \sqrt{\sum_{\tau=1}^{t-1}\Phi_\tau'(Q_\tau)^2}$. To bound (I-b), we observe that $\Phi_t'(Q_t) = \sqrt{(1 + \sum_{\tau=1}^t\Phi_\tau'(Q_\tau)^2)-(1+\sum_{\tau=1}^{t-1}\Phi_\tau'(Q_\tau)^2)}$. Then, by \Cref{prop:psi}, we have
    \begin{align*}
    \begin{aligned}
        \textrm{(I-b)} 
        &\leq 96\gamma_t G \sqrt{\log\frac{(\pi^2/6)(1 + \log_2(1+\sum_{\tau=1}^t\Phi_\tau'(Q_\tau)^2))^2}{\delta}}.
    \end{aligned}
    \end{align*}
    To further bound, we observe that $Q_{t+1} \leq \sum_{\tau=1}^tg_\tau(\bfx_\tau) \leq Gt$ for all $t \geq 1$ and $Q_1 = 0$, implying that $\sum_{\tau=1}^t\Phi_\tau'(Q_\tau)^2 \leq \sum_{\tau=1}^t \gamma_\tau^2e^{2\gamma_\tau G\tau}$. Note that $\gamma_\tau \leq 1/(12G\sqrt{\tau})$ and $\gamma_\tau \leq 1$. Then it follows that
    \[
        \sum_{\tau=1}^t \Phi_\tau'(Q_\tau)^2 \leq te^{\sqrt{t}/6}.
    \]
    Since $\log(1+t)\leq t$, it leads to
    \[
        \log_2\left(1+ \sum_{\tau=1}^t\Phi_\tau'(Q_\tau)^2\right)=\frac{\log\left(1+ te^{\sqrt{t}/6}\right)}{\log 2} \leq \frac{\log(1+t)+t/6}{\log 2} \leq \frac{t+t/6}{\log 2} = \frac{7t}{6\log2}.
    \]
    Hence, it follows that
    \[
        \log \frac{\pi^2}{6\delta}\left(1 + \log_2\left(1+\sum_{\tau=1}^t\Phi_\tau'(Q_\tau)^2\right)\right)^2 \leq \log \frac{\pi^2}{6\delta}\left(t + \frac{7t}{6\log2}\right)^2 \leq \log \frac{12t^2}{\delta}.
    \]
    Finally, we deduce that
    \begin{align}\label{eq:lem1:(I-b)}
        \textrm{(I-b)}\leq 96\gamma_t G\sqrt{\log\frac{12t^2}{\delta}}.
    \end{align}
    By substituting \eqref{eq:lem1:(I-a)} and \eqref{eq:lem1:(I-b)} into \eqref{eq:lem1:(I) decomp 2}, we have
    \begin{align}\label{eq:lem1:(I) decomp 3}
    \begin{aligned}
        3\gamma_t \frac{\Psi\left(\sum_{\tau=1}^t\Phi_\tau'(Q_\tau)^2\right)-\Psi\left(\sum_{\tau=1}^{t-1}\Phi_\tau'(Q_\tau)^2\right)}{\Phi_t'(Q_t)} 
        &\leq \gamma_t\left(12DL + 96G\sqrt{\log\frac{12t^2}{\delta}}\right)\\
        &\leq \frac{1}{2}
    \end{aligned}
    \end{align}
    where the second inequality follows from $\gamma_t \leq 1/(24(DL+8G\sqrt{\log(12t^2/\delta)}))$. Note that $\Psi$ is an increasing function, implying that $\Psi\left(\sum_{\tau=1}^t\Phi_\tau'(Q_\tau)^2\right)-\Psi\left(\sum_{\tau=1}^{t-1}\Phi_\tau'(Q_\tau)^2\right)$ is nonnegative. Then, by applying \eqref{eq:lem1:(I) decomp 1} and \eqref{eq:lem1:(I) decomp 3} to \eqref{eq:lem1:(I) decomp}, (I) is bounded as
    \begin{align*}
        &\textrm{(I)} = \Phi_{t}(Q_{t+1}) - \Phi_t(Q_t) \\
        &\leq \Phi_t'(Q_t) \left(g_t(\bfx_t)-R_t + 6\gamma_tG^2 +  \frac{\Psi\left(\sum_{\tau=1}^t\Phi_\tau'(Q_\tau)^2\right)-\Psi\left(\sum_{\tau=1}^{t-1}\Phi_\tau'(Q_\tau)^2\right)}{2\Phi_t'(Q_t)}\right) \\
        &= \Phi_t'(Q_t)\left(g_t(\bfx_t) - \frac{1}{2}R_t\right).
    \end{align*}
    Next, we bound (II). Since $\gamma_{t+1}\leq \gamma_t$, by \Cref{prop:Phi}, we have
    \[
    \textrm{(II)}=\Phi_{t+1}(Q_{t+1}) - \Phi_t(Q_{t+1}) \leq \frac{\gamma_t - \gamma_{t+1}}{e\gamma_{t+1}}.
    \]
    By applying bounds on (I) and (II) to \eqref{eq:lem1:(I)+(II):hp}, we conclude the proof.
\end{proof}

The following lemma shows that the cost incurred by using time-varying Lyapunov functions with the modified $\gamma_t$ defined in \eqref{eq:psi + gamma hp} is negligible in the sense that its summation scales as $\log t$.
\begin{lemma}\label{lem:gamma sum hp}
    Consider $\gamma_t$ given by \eqref{eq:psi + gamma hp}. Then, we have
    \[
        \sum_{\tau=1}^t \frac{\gamma_\tau - \gamma_{\tau+1}}{\gamma_{\tau+1}} \leq \frac{3(1+\log t)}{2}.
    \]
\end{lemma}
\begin{proof}
    Recall that $1/\gamma_\tau = \max\{12G\sqrt{\tau}, 24(DL+8G\sqrt{\log(12\tau^2/\delta)}), 1\}$. Note that $\max\{x,y,z\} - \max\{x',y',z\} \leq x-x' + y-y'$ for any $x > x'$ and $y > y'$. Moreover, we know that $1/\gamma_\tau \geq 12G\sqrt{\tau}$ and $1/\gamma_\tau \geq 192G\sqrt{\log (12\tau^2/\delta)}$. Then we have
    \begin{align*}
        \frac{\gamma_\tau - \gamma_{\tau+1}}{\gamma_{\tau+1}} 
        &= \frac{\frac{1}{\gamma_{\tau+1}} - \frac{1}{\gamma_\tau}}{\frac{1}{\gamma_\tau}}\\
        &\leq \frac{\sqrt{\tau+1} -\sqrt{\tau}}{\sqrt{\tau}} + \frac{\sqrt{\log(12(\tau+1)^2/\delta)} - \sqrt{\log(12\tau^2/\delta)}}{\sqrt{\log(12\tau^2/\delta)}} \\
        &\leq \frac{1}{2\tau} + \frac{\log(1 + 1/\tau)}{\log(12\tau^2/\delta)}\\
        &\leq \frac{1}{2\tau} + \frac{1}{\tau\log(12\tau^2/\delta)} \\
        &\leq \frac{3}{2\tau}
    \end{align*}
    where the last two inequality follow from $\log(1 + x) \leq x$ for any $x \geq 0$ and $\log (12\tau^2/\delta) \geq 1$, respectively. Hence, we have
    \begin{align*}
        \sum_{\tau=1}^t \frac{\gamma_\tau - \gamma_{\tau+1}}{\gamma_{\tau+1}}  \leq \sum_{\tau=1}^t \frac{3}{2\tau} \leq \frac{3(1+\log t)}{2}.
    \end{align*}
\end{proof}

\begin{lemma}\label{lem:Regret+Phi:stoc + hp + cvx}
    Consider the setting of \Cref{thm:stoc + hp + cvx}. Suppose that we run \Cref{alg:cumulative} with $\gamma_t,\Psi$ given by \eqref{eq:psi + gamma hp}, and $\eta_t$ given by \eqref{eq:eta cvx}. Then, with probability at least $1-\delta$, for all $t\geq 1$, we have
    \begin{align*}
        \sum_{\tau=1}^t (f_\tau(\bfx_\tau) - f_\tau(\bfx^\star)) + \Phi_{t+1}(Q_{t+1}) \leq 2DL\sqrt{t}
         + \log t+ 1 + D + 2G\sqrt{\log\frac{\pi^2}{6\delta}}.
    \end{align*}
\end{lemma}
\begin{proof}
    By \Cref{lem:adagrad}, we have
    \begin{align}\label{eq:Regret+Phi:1}
    \begin{aligned}
        &\sum_{\tau=1}^t (f_\tau(\bfx_\tau) - f_\tau(\bfx^\star)) + \sum_{\tau=1}^t \Phi_\tau'(Q_{\tau})g_\tau(\bfx_\tau)\\
        &\leq\sum_{\tau=1}^t\Phi_\tau'(Q_\tau)g_\tau(\bfx^\star) + 2DL\sqrt{t} + 2DL\sqrt{\sum_{\tau=1}^t \Phi_\tau'(Q_\tau)^2}
        + D.
    \end{aligned}
    \end{align}
    Moreover, the first term on the right-hand side can be bounded using \Cref{lem:supmg ci}. In particular, since $\Phi_\tau'(Q_\tau)$ is $\calF_{\tau-1}$-measurable and positive, $g_\tau(\bfx^\star)$ is $\calF_\tau$-measurable satisfying $|g_\tau(\bfx^\star)|\leq G$, and $\bbE[g_\tau(\bfx^\star)|\calF_{\tau-1}] = \bbE[g_\tau(\bfx^\star)] \leq 0$, it leads to the following: with probability at least $1-\delta$, for all $t\geq 1$,
    \begin{align}\label{eq:Regret+Phi:2}
    \begin{aligned}
        &\sum_{\tau=1}^t\Phi_\tau'(Q_\tau)g_\tau(\bfx^\star)\\
        &\leq 2G\sqrt{\left(1 + \sum_{\tau=1}^t\Phi_\tau'(Q_\tau)^2\right)\log\frac{(\pi^2/6)(1+\log_2(1+\sum_{\tau=1}^t\Phi_\tau'(Q_\tau)^2))^2}{\delta}}.
    \end{aligned}
    \end{align}
    Moreover, by summing the inequality in \Cref{lem:drift:stoc + hp + cvx} over $\tau=1, \ldots,t$, since $\Phi_1(Q_1) = 0$, we have
    \begin{align}\label{eq:Regret+Phi:3}
        \Phi_{t+1}(Q_{t+1}) - \Phi_1(Q_1) \leq \sum_{\tau=1}^t \Phi_\tau'(Q_\tau)\left(g_\tau(\bfx_\tau) - \frac{R_\tau}{2}\right) + \sum_{\tau=1}^t\frac{\gamma_\tau - \gamma_{\tau+1}}{e\gamma_{\tau+1}}.
    \end{align}
    By applying \eqref{eq:Regret+Phi:2} and \eqref{eq:Regret+Phi:3} to \eqref{eq:Regret+Phi:1}, we have
    \begin{align*}
        &\sum_{\tau=1}^t (f_\tau(\bfx_\tau)-f_\tau(\bfx^\star)) + \Phi_{t+1}(Q_{t+1}) \\
        &\leq 2G\sqrt{\left(1 + \sum_{\tau=1}^t\Phi_\tau'(Q_\tau)^2\right)\log\frac{(\pi^2/6)(1+\log_2(1+\sum_{\tau=1}^t\Phi_\tau'(Q_\tau)^2))^2}{\delta}} \\
        &\quad+ 2DL\sqrt{t} + 2DL\sqrt{\sum_{\tau=1}^t \Phi_\tau'(Q_\tau)^2}
        + D + \sum_{\tau=1}^t\frac{\gamma_\tau - \gamma_{\tau+1}}{e\gamma_{\tau+1}} - \frac{1}{2}\sum_{\tau=1}^t \Phi_\tau'(Q_\tau)R_\tau.
    \end{align*}
    Note that 
    \begin{align*}
        \frac{1}{2}\sum_{\tau=1}^t \Phi_\tau'(Q_\tau)R_\tau 
        &\geq \frac{1}{2}\sum_{\tau=1}^t \Phi_\tau'(Q_\tau)\cdot\frac{\Psi(\sum_{s=1}^\tau \Phi_s'(Q_s)^2)-\Psi(\sum_{s=1}^{\tau-1}\Phi_s'(Q_s)^2)}{\Phi_\tau'(Q_\tau)} \\
        &= \frac{1}{2}\Psi\left(\sum_{\tau=1}^t \Phi_\tau'(Q_\tau)^2\right) - \frac{1}{2}\Psi\left(0\right)\\
        &=2DL\sqrt{\sum_{\tau=1}^t \Phi_\tau'(Q_\tau)^2} - 2G\sqrt{\log\frac{\pi^2}{6\delta}} \\
        &\quad+ 2G\sqrt{\left(1+\sum_{\tau=1}^t\Phi_\tau'(Q_{\tau})^2\right)\log\frac{(\pi^2/6)(1 + \log_2(1+\sum_{\tau=1}^t\Phi_\tau'(Q_\tau)^2))^2}{\delta}}.
    \end{align*}
    Thus, we deduce that
    \begin{align*}
        \sum_{\tau=1}^t (f_\tau(\bfx_\tau)-f_\tau(\bfx^\star)) + \Phi_{t+1}(Q_{t+1}) 
        &\leq 2DL\sqrt{t}
         + \sum_{\tau=1}^t\frac{\gamma_\tau - \gamma_{\tau+1}}{e\gamma_{\tau+1}} + D+ 2G\sqrt{\log\frac{\pi^2}{6\delta}} \\
        &\leq 2DL\sqrt{t}
         + \log t+ 1 + D + 2G\sqrt{\log\frac{\pi^2}{6\delta}}
    \end{align*}
    where the second inequality follows from \Cref{lem:gamma sum hp} and $3/(2e) \leq1$. This concludes the proof.
\end{proof}

\begin{lemma}\label{lem:R sum:stoc + hp + cvx}
    Consider the setting of \Cref{thm:stoc + hp + cvx}. Suppose that we run \Cref{alg:cumulative} with $\gamma_t,\Psi$ given by \eqref{eq:psi + gamma hp}, and $\eta_t$ given by \eqref{eq:eta cvx}. Suppose that \Cref{lem:Regret+Phi:stoc + hp + cvx} holds. Then, for all $t\geq 1$, we have
    \begin{align*}
        \sum_{\tau=1}^t R_\tau =  \bigO\left(\sqrt{t}\log(t/\delta)\right).
    \end{align*}
\end{lemma}
\begin{proof}
    We first bound $\Phi_{t+1}'(Q_{t+1})$, which is frequently used throughout the proof. By Lipschitzness, we have $f_\tau(\bfx_\tau) - f_\tau(\bfx^\star) \geq -LD$. Then it leads to $\sum_{\tau=1}^t (f_\tau(\bfx_\tau) - f_\tau(\bfx^\star))\geq -LDt$. By \Cref{lem:Regret+Phi:stoc + hp + cvx}, we have
    \begin{align*}
        \Phi_{t+1}(Q_{t+1}) 
        &\leq 3DLt + \log t + 1 + D + 2G\sqrt{\log\frac{\pi^2}{6\delta}}.
    \end{align*}
    This leads to
    \begin{align}\label{eq:Phi' bound}
    \begin{aligned}
        \Phi_{t+1}'(Q_{t+1}) 
        &\leq \gamma_{t+1}\left(3DLt + \log t + 2 + D + 2G\sqrt{\log\frac{\pi^2}{6\delta}}\right)\\
        &\leq 3DLt + \log t + 2 + D + 2G\sqrt{\log\frac{\pi^2}{6\delta}}.
    \end{aligned}
    \end{align}
    Now, we bound $\sum_{\tau=1}^t R_\tau$. Since $\gamma_\tau \leq 1/(12G\sqrt{\tau})$, we have
    \begin{align*}
        \sum_{\tau=1}^t R_\tau \leq \sum_{\tau=1}^t\frac{G}{\sqrt{\tau}} + \sum_{\tau=1}^t\frac{\Psi\left(\sum_{s=1}^\tau\Phi_s'(Q_s)^2\right)-\Psi\left(\sum_{s=1}^{\tau-1}\Phi_s'(Q_s)^2\right)}{\Phi_\tau'(Q_\tau)}.
    \end{align*}
    It is clear that $\sum_{\tau=1}^t G/\sqrt{\tau} \leq 2G\sqrt{t}$. Moreover, it follows that
    \begin{align*}
        &\sum_{\tau=1}^t\frac{\Psi\left(\sum_{s=1}^\tau\Phi_s'(Q_s)^2\right)-\Psi\left(\sum_{s=1}^{\tau-1}\Phi_s'(Q_s)^2\right)}{\Phi_\tau'(Q_\tau)} \\
        &= \underbrace{4DL\sum_{\tau=1}^t \frac{\sqrt{ \sum_{s=1}^\tau\Phi_s'(Q_s)^2}-\sqrt{ \sum_{s=1}^{\tau-1}\Phi_s'(Q_s)^2}}{\Phi_\tau'(Q_\tau)}}_{\textrm{(I)}} \\
        &\quad+\underbrace{4G\sum_{\tau=1}^t \frac{\sqrt{\psi(\sum_{s=1}^\tau\Phi_s'(Q_s)^2)} - \sqrt{ \psi(\sum_{s=1}^{\tau-1}\Phi_s'(Q_s)^2)}}{\Phi_\tau'(Q_\tau)}}_{\textrm{(II)}}.
    \end{align*}
    We bound (I) as follows.
    \begin{align*}
        \textrm{(I)} 
        &= 4DL\sum_{\tau=1}^t\frac{\Phi_\tau'(Q_\tau)}{\sqrt{ \sum_{s=1}^\tau\Phi_s'(Q_s)^2}+\sqrt{ \sum_{s=1}^{\tau-1}\Phi_s'(Q_s)^2}}\\
        &\leq 4DL\sum_{\tau=1}^t\frac{\Phi_\tau'(Q_\tau)}{\sqrt{ \sum_{s=1}^\tau\Phi_s'(Q_s)^2}} \\
        &\leq 4DL\sqrt{t}\sqrt{\sum_{\tau=1}^t\frac{\Phi_\tau'(Q_\tau)^2}{\sum_{s=1}^\tau \Phi_s'(Q_s)^2}}\\
        &= 4DL\sqrt{t}\sqrt{\sum_{\tau=2}^t\frac{\sum_{s=1}^\tau\Phi_s'(Q_s)^2 - \sum_{s=1}^{\tau-1}\Phi_s'(Q_s)^2}{\sum_{s=1}^\tau \Phi_s'(Q_s)^2} + 1} \\
        &\leq 4DL\sqrt{t}\sqrt{\log\frac{\sum_{\tau=1}^t \Phi_\tau'(Q_\tau)^2}{\Phi_1'(Q_1)^2} +1}\\
        &\leq4DL\sqrt{t}\sqrt{\log\frac{t(3DLt + \log t + 2 + D + 2G\sqrt{\log(\pi^2/(6\delta))})^2}{\gamma_1^2} + 1}.
    \end{align*}
    where the second inequality follows from the Cauchy-Schwarz inequality, the third inequality is due to \Cref{lem:integral 1/x}, and the last inequality follows from \eqref{eq:Phi' bound}, i.e., since $\Phi_1'(Q_1)^2=\gamma_1^2 \leq 1$,
    \begin{align*}
        \sum_{\tau=1}^t \Phi_\tau'(Q_\tau)^2=\sum_{\tau=0}^{t-1} \Phi_{\tau+1}'(Q_{\tau+1})^2 
        &\leq 1+(t-1)(3DLt + \log t + 2 + D + 2G\sqrt{\log(\pi^2/(6\delta))})^2\\
        &\leq t(3DLt + \log t + 2 + D + 2G\sqrt{\log(\pi^2/(6\delta))})^2.
    \end{align*}
    To bound (II), by \Cref{prop:psi}, 
    \begin{align*}
        \textrm{(II)} 
        &\leq 32G\sum_{\tau=1}^t \sqrt{\log\frac{(\pi^2/6)(1+\log_2(1+\sum_{s=1}^\tau \Phi_s'(Q_s)^2))^2}{\delta}}\\
        &\quad\times\frac{\sqrt{1+ \sum_{s=1}^\tau \Phi_s'(Q_s)^2}-\sqrt{1+ \sum_{s=1}^{\tau-1} \Phi_s'(Q_s)^2}}{\Phi_\tau'(Q_\tau)} \\
        &\leq 32G\sqrt{\log\frac{(\pi^2/6)(1+\log_2(1+\sum_{\tau=1}^t \Phi_\tau'(Q_\tau)^2))^2}{\delta}}\\
        &\quad\times\sum_{\tau=1}^t \frac{\sqrt{1+ \sum_{s=1}^\tau \Phi_s'(Q_s)^2}-\sqrt{1+ \sum_{s=1}^{\tau-1} \Phi_s'(Q_s)^2}}{\Phi_\tau'(Q_\tau)} \\
        &\leq 32G\sqrt{t}\sqrt{\log\frac{(\pi^2/6)(1+\log_2(1+t(3DLt + \log t + 2 + D + 2G\sqrt{\log(\pi^2/(6\delta))})^2))^2}{\delta}}\\
        &\quad \times\sqrt{\log \frac{t(3DLt + \log t + 2 + D + 2G\sqrt{\log(\pi^2/(6\delta))})^2}{\gamma_1^2} + 1}.
    \end{align*}
    where the first inequality follows from \Cref{prop:psi}, and the third inequality follows from the same argument used for bounding (I). Hence, we deduce that
    \begin{align*}
        \sum_{\tau=1}^t R_\tau 
        &\leq 2G\sqrt{t} + 4DL\sqrt{t}\sqrt{\log\frac{t(3DLt + \log t + 2 + D + 2G\sqrt{\log(\pi^2/(6\delta))})^2}{\gamma_1^2} + 1} \\
        &\quad+32G\sqrt{t}\sqrt{\log\frac{(\pi^2/6)(1+\log_2(1+t(3DLt + \log t + 2 + D + 2G\sqrt{\log(\pi^2/(6\delta))})^2))^2}{\delta}}\\
        &\quad \times\sqrt{\log \frac{t(3DLt + \log t + 2 + D + 2G\sqrt{\log(\pi^2/(6\delta))})^2}{\gamma_1^2} + 1}.
    \end{align*}
    It can be written as $\sum_{\tau=1}^t R_\tau =  \bigO\left(\sqrt{t}\log(t/\delta)\right)$.
\end{proof}

\subsection{Proof of Theorem~\ref{thm:stoc + hp + cvx}}
\begin{proof}
    We assume that \Cref{lem:Regret+Phi:stoc + hp + cvx} holds, which occurs with probability at least $1-\delta$. Since $\Phi_{t+1}(Q_{t+1}) \geq -1$, we have
    \begin{align*}
        \Regret(t)=\sum_{\tau=1}^t (f_\tau(\bfx_\tau) - f_\tau(\bfx^\star))\leq 2DL\sqrt{t} + \log t + D + 2G\sqrt{\log\frac{\pi^2}{6\delta}} + 2 
    \end{align*}
    Next, to bound $\Violation(t)$, we observe that $\sum_{\tau=1}^t g_\tau(\bfx_\tau) = Q_{t+1} + \sum_{\tau=1}^t R_\tau$. Thus, we first bound $Q_{t+1}$. As done in the proof of \Cref{lem:R sum:stoc + hp + cvx}, we have
    \begin{align*}
    \begin{aligned} 
        \Phi_{t+1}(Q_{t+1}) 
        &\leq 3DLt + \log t + 1 + D + 2G\sqrt{\log\frac{\pi^2}{6\delta}}.
    \end{aligned}
    \end{align*}
    By the definition of $\Phi_{t+1}$, since $Q_{t+1} = (1/\gamma_{t+1})\log(\Phi_{t+1}(Q_{t+1})+1)$, we have
    \begin{align*}
        Q_{t+1} 
        &\leq \frac{1}{\gamma_{t+1}}\log\left(3DLt + \log t + 2 + D + 2G\sqrt{\log\frac{\pi^2}{6\delta}}\right)\\
        &\leq \left(12G\sqrt{t+1} + 24\left(DL+8G\sqrt{\log(12(t+1)^2/\delta)}\right) + 1\right) \\
        &\quad \times
        \log\left(3DLt + \log t + 2 + D + 2G\sqrt{\log\frac{\pi^2}{6\delta}}\right).
    \end{align*}
    It can be written as
    \[
        Q_{t+1} =  \bigO\left(\sqrt{t}\log(t/\delta) + \log^{3/2}(t/\delta)\right)
    \]
    Moreover, by \Cref{lem:R sum:stoc + hp + cvx}, we have $\sum_{\tau=1}^t R_\tau =  \bigO(\sqrt{t}\log(t/\delta))$. 
    Consequently, we have
    \[
        \Regret(t) =  \bigO\left(\sqrt{t} + \sqrt{\log(1/\delta)}\right), \quad \Violation(t) =  \bigO\left(\sqrt{t}\log(t/\delta) + \log^{3/2}(t/\delta)\right).
    \]
\end{proof}

\section{Analysis under Strongly Convex Losses}\label{sec:anal:stoc + strcvx}
In this section, we present the analysis of Theorems~\ref{thm:stoc + exp + strcvx} and~\ref{thm:stoc + hp + strcvx}. The key difference in the analysis is to derive an analogue of \Cref{lem:adagrad} with an improved guarantee under the modified primal step size defined in \eqref{eq:eta strcvx}. 
\begin{lemma}\label{lem:strcvx regret}
        Consider $\mu$-strongly convex losses for some $\mu>0$. Suppose that we run \Cref{alg:cumulative} with $\eta_\tau$ given by \eqref{eq:eta strcvx}. Then, for all $t\geq 1$, we have 
        \begin{align*}
        &\sum_{\tau=1}^t (f_\tau(\bfx_\tau) - f_\tau(\bfx^\star)) + \sum_{\tau=1}^t \Phi_\tau'(Q_{\tau})g_\tau(\bfx_\tau) \\
        &\leq \sum_{\tau=1}^t \Phi_\tau'(Q_\tau)g_\tau(\bfx^\star) + \frac{L^2(1+\log t)}{\mu} + 2DL\sqrt{\sum_{\tau=1}^t \Phi_\tau'(Q_\tau)^2}.
    \end{align*}
\end{lemma}
\begin{proof}
    For simplicity, let $\hat f_\tau = f_\tau + \Phi_\tau'(Q_\tau)g_\tau$ for each $\tau \in [t]$. Note that since $\Phi_\tau'(Q_\tau) > 0$, $\Phi_\tau'(Q_\tau)g_\tau$ is convex, and thus $\hat f_\tau$ is $\mu$-strongly convex by \Cref{lem:strcvx properties}. Since projection onto $\calX$ is a contraction mapping, we have for any $\bfx^\star \in \calX$,
    \begin{align*}
        \|\bfx_{\tau+1} - \bfx^\star\|^2 
        &\leq \|\bfx_\tau - \bfx^\star\|^2 + \eta_\tau^2\|\nabla \hat f_\tau(\bfx_\tau)\|^2 - 2\eta_\tau\nabla \hat f_\tau(\bfx_\tau)^\top (\bfx_\tau - \bfx^\star).
    \end{align*}
    By strong convexity (\Cref{lem:strcvx properties}), we have 
    \begin{align*}
        \hat f_\tau(\bfx_\tau) - \hat f_\tau(\bfx^\star) + \frac{\mu}{2}\|\bfx_\tau - \bfx^\star\|^2 \leq \nabla \hat f_\tau(\bfx_\tau)^\top (\bfx_\tau -\bfx^\star).
    \end{align*}
    Recall that
    \begin{align*}
        \eta_\tau = \frac{1}{\mu\tau + (2L/D)\sqrt{\sum_{s=1}^\tau \Phi_s'(Q_s)^2}}.
    \end{align*}
    Then it follows that
    \begin{align*}
        &\hat f_\tau(\bfx_\tau) - \hat f_\tau(\bfx^\star) \\
        &\leq \left(\frac{1}{2\eta_\tau} - \frac{\mu}{2}\right)\|\bfx_\tau -\bfx^\star\|^2 - \frac{1}{2\eta_\tau}\|\bfx_{\tau+1}-\bfx^\star\|^2 + \frac{\eta_\tau}{2}\|\nabla \hat f_\tau(\bfx_\tau)\|^2\\
        &= \frac{\mu(\tau-1)}{2}\|\bfx_\tau-\bfx^\star\|^2 - \frac{\mu \tau}{2}\|\bfx_{\tau+1}-\bfx^\star\|^2 \\
        &\quad+ \frac{L}{D}\sqrt{\sum_{s=1}^\tau \Phi_s'(Q_s)^2} (\|\bfx_\tau - \bfx^\star\|^2 - \|\bfx_{\tau+1} - \bfx^\star\|^2) + \frac{\eta_\tau}{2}\|\nabla \hat f_\tau(\bfx_\tau)\|^2\\
        &\leq \frac{\mu(\tau-1)}{2}\|\bfx_\tau-\bfx^\star\|^2 - \frac{\mu \tau}{2}\|\bfx_{\tau+1}-\bfx^\star\|^2 \\
        &\quad+ \frac{L}{D}\sqrt{\sum_{s=1}^\tau \Phi_s'(Q_s)^2} (\|\bfx_\tau - \bfx^\star\|^2 - \|\bfx_{\tau+1} - \bfx^\star\|^2) + \eta_\tau L^2 + \eta_\tau L^2\Phi_\tau'(Q_\tau)^2
    \end{align*}
    where the last inequality follows from that $\|\nabla \hat f_\tau(\bfx_\tau)\|^2 \leq 2\|\nabla f_\tau(\bfx_\tau)\|^2 + 2\Phi_\tau'(Q_\tau)^2\|\nabla g_\tau(\bfx_\tau)\|^2 \leq 2L^2(1+\Phi_\tau'(Q_\tau)^2)$.  By summing over $\tau = 1,\ldots, t$, we have
    \begin{align*}
        \sum_{\tau=1}^t(\hat f_\tau(\bfx_\tau) - \hat f_\tau(\bfx^\star)) 
        &\leq \frac{L}{D}\sum_{\tau=1}^t\sqrt{\sum_{s=1}^\tau \Phi_s'(Q_s)^2} (\|\bfx_\tau - \bfx^\star\|^2 - \|\bfx_{\tau+1} - \bfx^\star\|^2) \\
        &\quad+ L^2\sum_{\tau=1}^t\eta_\tau + L^2\sum_{\tau=1}^t\eta_\tau \Phi_\tau'(Q_\tau)^2.
    \end{align*}
    Note that 
    \begin{align*}
        \sum_{\tau=1}^t\sqrt{\sum_{s=1}^\tau \Phi_s'(Q_s)^2} (\|\bfx_\tau - \bfx^\star\|^2 - \|\bfx_{\tau+1} - \bfx^\star\|^2) 
        &\leq D^2\sqrt{\sum_{\tau=1}^t \Phi_\tau'(Q_\tau)^2}.
    \end{align*}
    Moreover, since $\sum_{\tau=1}^t (1/\tau) \leq 1+\log t$, we have
    \begin{align*}
        \sum_{\tau=1}^t \eta_\tau \leq \sum_{\tau=1}^t \frac{1}{\mu\tau} \leq \frac{1 + \log t}{\mu}.
    \end{align*}
    Similarly, since $\sum_{\tau=1}^t y_\tau / \sqrt{\sum_{s=1}^\tau y_s} \leq 2\sqrt{\sum_{\tau=1}^t y_\tau}$ for any positive sequence $\{y_\tau\}_{\tau\geq 1}$, we have
    \begin{align*}
        \sum_{\tau=1}^t \eta_\tau \Phi_\tau'(Q_\tau)^2 \leq \frac{D}{2L}\sum_{\tau=1}^t \frac{\Phi_\tau'(Q_\tau)^2}{\sqrt{\sum_{s=1}^\tau \Phi_s'(Q_s)^2}} \leq \frac{D}{L} \sqrt{\sum_{\tau=1}^t \Phi_\tau'(Q_\tau)^2}.
    \end{align*}
    Combining these yields
    \begin{align*}
        \sum_{\tau=1}^t(\hat f_\tau(\bfx_\tau) - \hat f_\tau(\bfx^\star)) \leq  \frac{L^2(1+\log t)}{\mu} + 2DL\sqrt{\sum_{\tau=1}^t \Phi_\tau'(Q_\tau)^2}.
    \end{align*}
    Recall that $\hat f_\tau = f_\tau + \Phi_\tau'(Q_\tau)g_\tau$. Therefore, we have
    \begin{align*}
        &\sum_{\tau=1}^t (f_\tau(\bfx_\tau) - f_\tau(\bfx^\star)) + \sum_{\tau=1}^t \Phi_\tau'(Q_{\tau})g_\tau(\bfx_\tau) \\
        &\leq \sum_{\tau=1}^t \Phi_\tau'(Q_\tau)g_\tau(\bfx^\star) + \frac{L^2(1 + \log t)}{\mu} + 2DL\sqrt{\sum_{\tau=1}^t \Phi_\tau'(Q_\tau)^2}.
    \end{align*}
\end{proof}

\subsection{Expected Bounds}
The proofs of the following lemmas are deferred to Appendix~\ref{app:deferred}.
\begin{lemma}\label{lem:Regret+Phi:stoc + exp + strcvx}
    Consider the setting of \Cref{thm:stoc + exp + strcvx}. Suppose that we run \Cref{alg:cumulative} with $\gamma_t,\Psi$ given by \eqref{eq:psi + gamma exp}, and $\eta_t$ given by \eqref{eq:eta strcvx}. Then, for all $t\geq 1$, we have
    \begin{align*}
        \bbE\left[\sum_{\tau=1}^t (f_\tau(\bfx_\tau) - f_\tau(\bfx^\star))\right] + \bbE\left[\Phi_{t+1}(Q_{t+1})\right] \leq \frac{L^2(1+\log t)}{\mu} + \log t + 1.
    \end{align*}
\end{lemma}

\begin{lemma}\label{lem:R sum:stoc + exp + strcvx}
    Consider the setting of \Cref{thm:stoc + exp + strcvx}. Suppose that we run \Cref{alg:cumulative} with $\gamma_t,\Psi$ given by \eqref{eq:psi + gamma exp}, and $\eta_t$ given by \eqref{eq:eta strcvx}. Then, for all $t\geq 1$, we have
    \begin{align*}
        \bbE\left[\sum_{\tau=1}^t R_\tau\right] =  \bigO\left(\sqrt{t \log t}\right).
    \end{align*}
\end{lemma}

\subsubsection{Proof of Theorem~\ref{thm:stoc + exp + strcvx}}
\begin{proof}
    We closely follow the proof of \Cref{thm:stoc + exp + cvx}, where the only difference is that we use \Cref{lem:Regret+Phi:stoc + exp + strcvx} and \Cref{lem:R sum:stoc + exp + strcvx}. By \Cref{lem:Regret+Phi:stoc + exp + strcvx}, 
    \begin{align*}
        \bbE[\Regret(t)] = \bbE\left[\sum_{\tau=1}^t (f_\tau(\bfx_\tau) - f_\tau(\bfx^\star))\right] \leq \frac{L^2(1+\log t)}{\mu} + \log t + 2.
    \end{align*}
    By applying the same argument in the proof of \Cref{thm:stoc + exp + cvx}, we have
    \begin{align*}
        \bbE[Q_{t+1}] 
        &\leq \frac{1}{\gamma_{t+1}}\log\left(DLt + \frac{L^2(1+\log t)}{\mu} + \log t + 1\right)\\
        &\leq \left(12G\sqrt{t+1} + 24DL + 1\right)\log\left(DLt + \frac{L^2(1+\log t)}{\mu} + \log t + 1\right)
    \end{align*}
    which can be written as $\bbE[Q_{t+1}] = \bigO\left(\sqrt{t}\log t\right)$. Moreover, by \Cref{lem:R sum:stoc + exp + strcvx}, we have $\bbE\left[\sum_{\tau=1}^t R_\tau\right] =\bigO(\sqrt{t\log t})$. Since $\bbE[\Violation(t)] = \bbE[Q_{t+1}] + \bbE\left[\sum_{\tau=1}^t R_\tau\right]$, we have
    \[
        \bbE[\Violation(t)] = \bigO\left(\sqrt{t}\log t\right).
    \]
\end{proof}

\subsection{High-Probability Bounds}
The proofs of the following lemmas are deferred to Appendix~\ref{app:deferred}.
\begin{lemma}\label{lem:Regret+Phi:stoc + hp + strcvx}
    Consider the setting of \Cref{thm:stoc + hp + strcvx}. Suppose that we run \Cref{alg:cumulative} with $\gamma_t,\Psi$ as given by \eqref{eq:psi + gamma hp}, and $\eta_t$ as given by \eqref{eq:eta strcvx}. Then, with probability at least $1-\delta$, for all $t\geq 1$, we have
    \begin{align*}
        \sum_{\tau=1}^t (f_\tau(\bfx_\tau) - f_\tau(\bfx^\star)) + \Phi_{t+1}(Q_{t+1}) \leq \frac{L^2(1+\log t)}{\mu} + \log t + 1 + 2G\sqrt{\log\frac{\pi^2}{6\delta}}.
    \end{align*}
\end{lemma}

\begin{lemma}\label{lem:R sum:stoc + hp + strcvx}
    Consider the setting of \Cref{thm:stoc + hp + strcvx}. Suppose that we run \Cref{alg:cumulative} with $\gamma_t,\Psi$ as given by \eqref{eq:psi + gamma hp}, and $\eta_t$ as given by \eqref{eq:eta strcvx}. Suppose that \Cref{lem:Regret+Phi:stoc + hp + strcvx} holds. Then, for all $t\geq 1$, we have
    \begin{align*}
        \sum_{\tau=1}^t R_\tau =  \bigO\left(\sqrt{t}\log(t/\delta)\right).
    \end{align*}
\end{lemma}

\subsubsection{Proof of Theorem~\ref{thm:stoc + hp + strcvx}}
\begin{proof}
    We closely follow the proof of \Cref{thm:stoc + hp + cvx}, where the only difference is that we use \Cref{lem:Regret+Phi:stoc + hp + strcvx} and \Cref{lem:R sum:stoc + hp + strcvx}. We assume that \Cref{lem:Regret+Phi:stoc + hp + strcvx} holds, which occurs with probability at least $1-\delta$. By \Cref{lem:Regret+Phi:stoc + hp + strcvx}, 
    \begin{align*}
        \Regret(t) = \sum_{\tau=1}^t (f_\tau(\bfx_\tau) - f_\tau(\bfx^\star)) \leq \frac{L^2(1+\log t)}{\mu} + \log t + 2 + 2G\sqrt{\log\frac{\pi^2}{6\delta}}.
    \end{align*}
    By applying the same argument in the proof of \Cref{thm:stoc + hp + cvx}, we have
    \begin{align*}
        Q_{t+1}
        &\leq \frac{1}{\gamma_{t+1}}\log\left(DLt + \frac{L^2(1+\log t)}{\mu} + \log t + 2 + 2G\sqrt{\log\frac{\pi^2}{6\delta}}\right)\\
        &\leq \left(12G\sqrt{t+1} + 24\left(DL+8G\sqrt{\log(12(t+1)^2/\delta)}\right) + 1\right)\\
        &\quad\times\log\left(DLt + \frac{L^2(1+\log t)}{\mu} + \log t + 2 + 2G\sqrt{\log\frac{\pi^2}{6\delta}}\right)
    \end{align*}
    which can be written as $Q_{t+1} = \bigO\left(\sqrt{t}\log t + \log^{3/2} (t/\delta)\right)$. Moreover, by \Cref{lem:R sum:stoc + hp + strcvx}, we have $\sum_{\tau=1}^t R_\tau=\bigO(\sqrt{t}\log(t/\delta))$. Since $\Violation(t) = Q_{t+1} + \sum_{\tau=1}^t R_\tau$, we have
    \[
        \Violation(t) = \bigO\left(\sqrt{t}\log (t/\delta) + \log^{3/2}(t/\delta)\right),
    \]
    as required.
\end{proof}

\section{Conclusion}\label{sec:conclusion}
In this paper, we study COCO with stochastic constraints without Slater's condition. We propose a unified primal-dual algorithm with an anytime regularizer and show that it achieves nearly optimal regret and constraint violation guarantees under stochastic constraints. The key insight is that the regularizer stabilizes the dual process by offsetting an accumulated term involving the derivative of an exponential Lyapunov function. We also establish high-probability guarantees and extend the analysis to strongly convex losses and adversarial constraints.
Several directions are worth further investigation. One direction is to obtain dynamic regret bounds under stochastic constraints, where dynamic regret is defined with respect to a comparator sequence $\{\bfx_t^\star\}_{t=1}^T$ satisfying $\{\bfx_t^\star\}_{t=1}^T \in \argmin_{\{\bfx_t\}_{t=1}^T} \sum_{t=1}^T f_t(\bfx_t) \st \bbE[g_t(\bfx_t)] \leq 0$ for all $t\in[T]$, while this paper considers a fixed comparator over time. Another direction is to develop projection-free algorithms, as the proposed algorithm requires a projection oracle onto $\calX$.


\newpage
\appendix
\newpage

\section{Deferred Proofs of Section~\ref{sec:anal:stoc + strcvx}}\label{app:deferred}

\begin{proof}[Proof of Lemma~\ref{lem:Regret+Phi:stoc + exp + strcvx}]
    We closely follow the proof of \Cref{lem:Regret+Phi:stoc + exp + cvx}, where the only difference is that we use \Cref{lem:strcvx regret} instead of \Cref{lem:adagrad}. Since we take $\eta_t$ as \eqref{eq:eta strcvx} and $\gamma_t, \Psi$ as \eqref{eq:psi + gamma exp}, Lemmas~\ref{lem:strcvx regret} and~\ref{lem:drift:stoc + exp + cvx} hold. By \Cref{lem:strcvx regret}, we have
    \begin{align*}
    \begin{aligned}
        &\sum_{\tau=1}^t (f_\tau(\bfx_\tau) - f_\tau(\bfx^\star)) + \sum_{\tau=1}^t \Phi_\tau'(Q_{\tau})g_\tau(\bfx_\tau) \\
        &\leq \sum_{\tau=1}^t \Phi_\tau'(Q_\tau)g_\tau(\bfx^\star) + \frac{L^2(1+\log t)}{\mu} + 2DL\sqrt{\sum_{\tau=1}^t \Phi_\tau'(Q_\tau)^2}.
    \end{aligned}
    \end{align*}
    By \Cref{lem:drift:stoc + exp + cvx},
    \begin{align*}
        \Phi_{t+1}(Q_{t+1}) - \Phi_1(Q_1) \leq \sum_{\tau=1}^t \Phi_\tau'(Q_\tau)\left(g_\tau(\bfx_\tau) - \frac{R_\tau}{2}\right) + \sum_{\tau=1}^t\frac{\gamma_\tau - \gamma_{\tau+1}}{e\gamma_{\tau+1}}.
    \end{align*}
    Note that
    \begin{align*}
        \frac{1}{2}\sum_{\tau=1}^t \Phi_\tau'(Q_\tau)R_\tau 
        &\geq 2DL\sqrt{\sum_{\tau=1}^t \Phi_\tau'(Q_\tau)^2}.
    \end{align*}
    Then it follows that
    \begin{align*}
        &\sum_{\tau=1}^t (f_\tau(\bfx_\tau)-f_\tau(\bfx^\star)) + \Phi_{t+1}(Q_{t+1}) \\
        &\leq\sum_{\tau=1}^t\Phi_\tau'(Q_\tau)g_\tau(\bfx^\star) + \frac{L^2(1+\log t)}{\mu} + \sum_{\tau=1}^t\frac{\gamma_\tau - \gamma_{\tau+1}}{e\gamma_{\tau+1}} \\
        &\leq\sum_{\tau=1}^t\Phi_\tau'(Q_\tau)g_\tau(\bfx^\star) + \frac{L^2(1+\log t)}{\mu} + \log t + 1
    \end{align*}
    where the second inequality follows from \Cref{lem:gamma sum exp} and $1/(2e) \leq 1$. By taking $\bbE$ on both sides, since $\sum_{\tau=1}^t\bbE[\Phi_\tau'(Q_\tau)g_\tau(\bfx^\star)] \leq  0$ by applying the same argument in the proof of \Cref{lem:Regret+Phi:stoc + exp + cvx}, we conclude the proof.
\end{proof}

\begin{proof}[Proof of Lemma~\ref{lem:R sum:stoc + exp + strcvx}]
    We closely follow the proof of \Cref{lem:R sum:stoc + exp + cvx}, where the only difference is that we use \Cref{lem:Regret+Phi:stoc + exp + strcvx} instead of \Cref{lem:Regret+Phi:stoc + exp + cvx}. By Lipschitzness and \Cref{lem:Regret+Phi:stoc + exp + strcvx}, 
    \begin{align*}
        \bbE[\Phi_{t+1}'(Q_{t+1})] 
        &\leq DLt + \frac{L^2(1+\log t)}{\mu} + \log t + 1.
    \end{align*}
    Note that
    \begin{align*}
        \bbE\left[\sum_{\tau=1}^t R_\tau\right] 
        \leq \sum_{\tau=1}^t\frac{G}{\sqrt{\tau}} + \sum_{\tau=1}^t
        \bbE\left[\frac{\Psi\left(\sum_{s=1}^\tau\Phi_s'(Q_s)^2\right)-\Psi\left(\sum_{s=1}^{\tau-1}\Phi_s'(Q_s)^2\right)}{\Phi_\tau'(Q_\tau)}\right].
    \end{align*}
     It is clear that $\sum_{\tau=1}^t G/\sqrt{\tau} \leq 2G\sqrt{t}$. Moreover, it follows that
     \begin{align*}
         &\sum_{\tau=1}^t\bbE\left[\frac{\Psi\left(\sum_{s=1}^\tau\Phi_s'(Q_s)^2\right)-\Psi\left(\sum_{s=1}^{\tau-1}\Phi_s'(Q_s)^2\right)}{\Phi_\tau'(Q_\tau)}\right]\\
         &\leq 4DL\sqrt{t}\sqrt{2\log\frac{\sum_{\tau=1}^t\bbE[\Phi_\tau'(Q_\tau)]}{\Phi_1'(Q_1)} +1}\\
         &\leq 4DL\sqrt{t}\sqrt{2\log\frac{t(DLt + L^2(1+\log t)/\mu +\log t + 1)}{\Phi_1'(Q_1)} +1}.
     \end{align*}
     Finally, we have $\bbE\left[\sum_{\tau=1}^t R_\tau\right]  = \bigO\left(\sqrt{t \log t}\right).$
\end{proof}

\begin{proof}[Proof of Lemma~\ref{lem:Regret+Phi:stoc + hp + strcvx}]
    We closely follow the proof of \Cref{lem:Regret+Phi:stoc + hp + cvx}, where the only difference is that we use \Cref{lem:strcvx regret} instead of \Cref{lem:adagrad}. Since we take $\eta_t$ as \eqref{eq:eta strcvx} and $\gamma_t, \Psi$ as \eqref{eq:psi + gamma hp}, Lemmas~\ref{lem:strcvx regret} and~\ref{lem:drift:stoc + hp + cvx} hold. By \Cref{lem:strcvx regret}, we have
    \begin{align*}
    \begin{aligned}
        &\sum_{\tau=1}^t (f_\tau(\bfx_\tau) - f_\tau(\bfx^\star)) + \sum_{\tau=1}^t \Phi_\tau'(Q_{\tau})g_\tau(\bfx_\tau) \\
        &\leq \sum_{\tau=1}^t \Phi_\tau'(Q_\tau)g_\tau(\bfx^\star) + \frac{L^2(1+\log t)}{\mu} + 2DL\sqrt{\sum_{\tau=1}^t \Phi_\tau'(Q_\tau)^2}.
    \end{aligned}
    \end{align*}
    By \Cref{lem:drift:stoc + hp + cvx},
    \begin{align*}
        \Phi_{t+1}(Q_{t+1}) - \Phi_1(Q_1) \leq \sum_{\tau=1}^t \Phi_\tau'(Q_\tau)\left(g_\tau(\bfx_\tau) - \frac{R_\tau}{2}\right) + \sum_{\tau=1}^t\frac{\gamma_\tau - \gamma_{\tau+1}}{e\gamma_{\tau+1}}.
    \end{align*}
    Note that
        \begin{align*}
        \frac{1}{2}\sum_{\tau=1}^t \Phi_\tau'(Q_\tau)R_\tau 
        &\geq2DL\sqrt{\sum_{\tau=1}^t \Phi_\tau'(Q_\tau)^2} - 2G\sqrt{\log\frac{\pi^2}{6\delta}} \\
        &\quad+ 2G\sqrt{\left(1+\sum_{\tau=1}^t\Phi_\tau'(Q_{\tau})^2\right)\log\frac{(\pi^2/6)(1 + \log_2(1+\sum_{\tau=1}^t\Phi_\tau'(Q_\tau)^2))^2}{\delta}}.
    \end{align*}
    Then it follows that with probability at least $1-\delta$,
    \begin{align*}
        \sum_{\tau=1}^t (f_\tau(\bfx_\tau)-f_\tau(\bfx^\star)) + \Phi_{t+1}(Q_{t+1}) 
        &\leq \frac{L^2(1+\log t)}{\mu} + \sum_{\tau=1}^t\frac{\gamma_\tau - \gamma_{\tau+1}}{e\gamma_{\tau+1}} + 2G\sqrt{\log\frac{\pi^2}{6\delta}}\\
        &\leq \frac{L^2(1+\log t)}{\mu} + \log t + 1 + 2G\sqrt{\log\frac{\pi^2}{6\delta}}\\
    \end{align*}
    where the second inequality follows from \Cref{lem:gamma sum hp} and $3/(2e) \leq 1$.
\end{proof}

\begin{proof}[Proof of Lemma~\ref{lem:R sum:stoc + hp + strcvx}]
    We closely follow the proof of \Cref{lem:R sum:stoc + hp + cvx}, where the only difference is that we use \Cref{lem:Regret+Phi:stoc + hp + strcvx} instead of \Cref{lem:Regret+Phi:stoc + hp + cvx}. By Lipschitzness and \Cref{lem:Regret+Phi:stoc + hp + strcvx}, 
    \begin{align*}
        \Phi_{t+1}'(Q_{t+1})
        &\leq DLt + \frac{L^2(1+\log t)}{\mu} + \log t + 2 + 2G\sqrt{\log\frac{\pi^2}{6\delta}}.
    \end{align*}
    Note that
    \begin{align*}
        \sum_{\tau=1}^t R_\tau \leq \sum_{\tau=1}^t\frac{G}{\sqrt{\tau}} + \sum_{\tau=1}^t\frac{\Psi\left(\sum_{s=1}^\tau\Phi_s'(Q_s)^2\right)-\Psi\left(\sum_{s=1}^{\tau-1}\Phi_s'(Q_s)^2\right)}{\Phi_\tau'(Q_\tau)}.
    \end{align*}
     It is clear that $\sum_{\tau=1}^t G/\sqrt{\tau} \leq 2G\sqrt{t}$. Moreover, by applying the same argument in the proof of \Cref{lem:R sum:stoc + hp + cvx},
     \begin{align*}
         &\sum_{\tau=1}^t\frac{\Psi\left(\sum_{s=1}^\tau\Phi_s'(Q_s)^2\right)-\Psi\left(\sum_{s=1}^{\tau-1}\Phi_s'(Q_s)^2\right)}{\Phi_\tau'(Q_\tau)}\\
         &\leq\left(4DL + 32G\sqrt{\log\frac{(\pi^2/6)(1+\log_2(1+\sum_{\tau=1}^t \Phi_\tau'(Q_\tau)^2))^2}{\delta}}\right)\sqrt{t}\sqrt{\log\frac{\sum_{\tau=1}^t \Phi_\tau'(Q_\tau)^2}{\Phi_1'(Q_1)^2} +1}.
     \end{align*}
     Note that since $\Phi_{t+1}'(Q_{t+1}) = \bigO\left(t + \sqrt{\log(1/\delta)}\right)$, we have $\sum_{\tau=1}^t \Phi_\tau'(Q_\tau)^2 = \bigO\left(t^3 + t\log(1/\delta)\right)$. Therefore, we have
     \begin{align*}
         \sum_{\tau=1}^t\frac{\Psi\left(\sum_{s=1}^\tau\Phi_s'(Q_s)^2\right)-\Psi\left(\sum_{s=1}^{\tau-1}\Phi_s'(Q_s)^2\right)}{\Phi_\tau'(Q_\tau)} = \bigO\left(\sqrt{t}\log (t/\delta)\right).
     \end{align*}
     Consequently, we have $\sum_{\tau=1}^t R_\tau = \bigO\left(\sqrt{t}\log(t/\delta)\right)$.
\end{proof}

\section{Analysis under Adversarial Constraints}\label{app:anal:adv}
\subsection{Convex Losses}

\begin{lemma}\label{lem:drift:adv + exp + cvx}
Suppose that we run \Cref{alg:hard} with $\gamma_t,\Psi$ as given by \eqref{eq:psi + gamma exp}. Then, for all $t\geq 1$, we have
    \begin{equation*} 
        \Phi_{t+1}(Q_{t+1}) - \Phi_t(Q_{t}) \leq \Phi_t'(Q_{t})\left( \tilde g_t(\bfx_t) - \frac{R_t}{2}\right) + \frac{\gamma_t - \gamma_{t+1}}{e\gamma_{t+1}}.
    \end{equation*}
\end{lemma}
\begin{proof}
    Note that $|\tilde g_t(\bfx_t)| \leq |g_t(\bfx_t)| \leq G$. Therefore, by applying the same argument as in the proof of \Cref{lem:drift:stoc + exp + cvx}, we conclude the proof.
\end{proof}

\begin{lemma}\label{lem:adagrad adv}
        Consider the setting of \Cref{thm:adv + exp + cvx}. Suppose that we run \Cref{alg:hard} with $\eta_t$ given by \eqref{eq:eta cvx}. Then, for all $t\geq 1$, we have
        \begin{align*}
            &\sum_{\tau=1}^t (f_\tau(\bfx_\tau) - f_\tau(\bfx^\star)) + \sum_{\tau=1}^t \Phi_\tau'(Q_{\tau})\tilde g_\tau(\bfx_\tau) \\
            &\leq \sum_{\tau=1}^t \Phi_\tau'(Q_\tau)\tilde g_\tau(\bfx^\star) + 2DL \sqrt{t}+ 2DL \sqrt{\sum_{\tau=1}^t\Phi_\tau'(Q_\tau)^2} + D.
        \end{align*}
\end{lemma}
\begin{proof}
    Recall that $\nabla \tilde g_t(\bfx_t) = \nabla g_t(\bfx_t)$ if $g_t(\bfx_t)>0$, and $\bm 0$ otherwise. It follows that $\|\nabla \tilde g_t(\bfx_t)\|\leq L$. Moreover, we have that $\nabla f_t(\bfx_t) + \Phi_t'(Q_t)\nabla \tilde g_t(\bfx_t)$ is a subgradient of $f_t + \Phi_t'(Q_t)\tilde g_t$. Then, by applying the same argument as in \Cref{lem:adagrad}, we conclude the proof.
\end{proof}

\begin{lemma}\label{lem:Regret+Phi:adv + exp + cvx}
    Consider the setting of \Cref{thm:adv + exp + cvx}. Suppose that we run \Cref{alg:hard} with $\gamma_t,\Psi$ given by \eqref{eq:psi + gamma exp}, and $\eta_t$ given by \eqref{eq:eta cvx}. Then, for all $t\geq 1$, we have
    \begin{align*}
        \sum_{\tau=1}^t (f_\tau(\bfx_\tau) - f_\tau(\bfx^\star))+ \Phi_{t+1}(Q_{t+1}) \leq 2DL\sqrt{t} +  \log t + 1 + D.
    \end{align*}
\end{lemma}
\begin{proof}
    We closely follow the proof of \Cref{lem:Regret+Phi:stoc + exp + cvx}. 
    By \Cref{lem:adagrad adv},
    \begin{align*}
            &\sum_{\tau=1}^t (f_\tau(\bfx_\tau) - f_\tau(\bfx^\star)) + \sum_{\tau=1}^t \Phi_\tau'(Q_{\tau})\tilde g_\tau(\bfx_\tau) \\
            &\leq \sum_{\tau=1}^t \Phi_\tau'(Q_\tau)\tilde g_\tau(\bfx^\star) + 2DL \sqrt{t}+ 2DL \sqrt{\sum_{\tau=1}^t\Phi_\tau'(Q_\tau)^2} + D.
    \end{align*}
    By \Cref{lem:drift:adv + exp + cvx},
    \begin{align*}
        \Phi_{t+1}(Q_{t+1}) - \Phi_1(Q_1) \leq \sum_{\tau=1}^t \Phi_\tau'(Q_\tau)\left(\tilde g_\tau(\bfx_\tau) - \frac{R_\tau}{2}\right) + \sum_{\tau=1}^t\frac{\gamma_\tau - \gamma_{\tau+1}}{e\gamma_{\tau+1}}.
    \end{align*}
    By \Cref{lem:gamma sum exp},
    \begin{align*}
        \sum_{\tau=1}^t\frac{\gamma_\tau - \gamma_{\tau+1}}{e\gamma_{\tau+1}} \leq \frac{1+\log t}{2}
    \end{align*}
    Combining these results yields
    \begin{align*}
        \sum_{\tau=1}^t (f_\tau(\bfx_\tau)-f_\tau(\bfx^\star)) + \Phi_{t+1}(Q_{t+1}) 
        &\leq \sum_{\tau=1}^t\Phi_\tau'(Q_\tau)\tilde g_\tau(\bfx^\star)+2DL\sqrt{t}
         + \log t + 1 + D.
    \end{align*}
    Since we have $g_t(\bfx^\star) \leq 0$ for all $t\geq 1$ in the adversarial constraint setting, we know that $\sum_{\tau=1}^t \Phi_\tau'(Q_\tau) \tilde g_\tau(\bfx^\star) = 0$. This concludes the proof.  
\end{proof}

\begin{lemma}\label{lem:R sum:adv + exp + cvx}
    Consider the setting of \Cref{thm:adv + exp + cvx}. Suppose that we run \Cref{alg:hard} with $\gamma_t,\Psi$ given by \eqref{eq:psi + gamma exp}, and $\eta_t$ given by \eqref{eq:eta cvx}. Then, for all $t\geq 1$, we have
    \begin{align*}
        \sum_{\tau=1}^t R_\tau =  \bigO\left(\sqrt{t \log t}\right).
    \end{align*}
\end{lemma}
\begin{proof}
    We closely follow the proof of \Cref{lem:R sum:stoc + exp + cvx}, where the only difference is that we use \Cref{lem:Regret+Phi:adv + exp + cvx} instead of \Cref{lem:Regret+Phi:stoc + exp + cvx}. By Lipschitzness and \Cref{lem:Regret+Phi:adv + exp + cvx}, 
    \begin{align*}
        \Phi_{t+1}'(Q_{t+1})
        &\leq 3DLt +  \log t + 2 + D.
    \end{align*}
    Note that
    \begin{align*}
        \sum_{\tau=1}^t R_\tau 
        \leq \sum_{\tau=1}^t\frac{G}{\sqrt{\tau}} + \sum_{\tau=1}^t
        \frac{\Psi\left(\sum_{s=1}^\tau\Phi_s'(Q_s)^2\right)-\Psi\left(\sum_{s=1}^{\tau-1}\Phi_s'(Q_s)^2\right)}{\Phi_\tau'(Q_\tau)}.
    \end{align*}
     It is clear that $\sum_{\tau=1}^t G/\sqrt{\tau} \leq 2G\sqrt{t}$. Moreover, it follows that
     \begin{align*}
         &\sum_{\tau=1}^t\frac{\Psi\left(\sum_{s=1}^\tau\Phi_s'(Q_s)^2\right)-\Psi\left(\sum_{s=1}^{\tau-1}\Phi_s'(Q_s)^2\right)}{\Phi_\tau'(Q_\tau)}\\
         &\leq 4DL\sqrt{t}\sqrt{2\log\frac{\sum_{\tau=1}^t\Phi_\tau'(Q_\tau)}{\Phi_1'(Q_1)} +1}\\
         &\leq 4DL\sqrt{t}\sqrt{2\log\frac{t(3DLt +  \log t + 1 + D)}{\Phi_1'(Q_1)} +1}.
     \end{align*}
     Finally, we have $\sum_{\tau=1}^t R_\tau = \bigO\left(\sqrt{t \log t}\right).$
\end{proof}

\subsubsection{Proof of Theorem~\ref{thm:adv + exp + cvx}}
\begin{proof}
    Under \Cref{alg:hard}, we observe that $\Violation_+(t) = \sum_{\tau=1}^t [g_\tau(\bfx_\tau)]_+ = Q_{t+1} + \sum_{\tau=1}^t R_\tau$. 
    By \Cref{lem:Regret+Phi:adv + exp + cvx}, 
    \begin{align*}
        \Regret(t) = \sum_{\tau=1}^t (f_\tau(\bfx_\tau) - f_\tau(\bfx^\star)) \leq 2DL\sqrt{t} +  \log t + 2 + D.
    \end{align*}
    By applying the same argument in the proof of \Cref{thm:stoc + exp + cvx}, we have
    \begin{align*}
        Q_{t+1}
        &\leq \frac{1}{\gamma_{t+1}}\log\left(3DLt +  \log t + 1 + D\right)\\
        &\leq \left(12G\sqrt{t+1} + 24DL + 1\right)\log\left(3DLt +  \log t + 1 + D\right)
    \end{align*}
    which can be written as $Q_{t+1} = \bigO\left(\sqrt{t}\log t\right)$. Moreover, by \Cref{lem:R sum:adv + exp + cvx}, we have $\sum_{\tau=1}^t R_\tau =\bigO(\sqrt{t\log t})$. Since $\Violation_+(t) = Q_{t+1} + \sum_{\tau=1}^t R_\tau$, we have
    \[
        \Violation_+(t) = \bigO\left(\sqrt{t}\log t\right).
    \]
\end{proof}

\subsection{Strongly Convex Losses}

\begin{lemma}\label{lem:strcvx regret adv}
        Consider the setting of \Cref{thm:adv + exp + strcvx}. Suppose that we run \Cref{alg:hard} with $\eta_\tau$ given by \eqref{eq:eta strcvx}. Then, for all $t\geq 1$, we have 
        \begin{align*}
        &\sum_{\tau=1}^t (f_\tau(\bfx_\tau) - f_\tau(\bfx^\star)) + \sum_{\tau=1}^t \Phi_\tau'(Q_{\tau})\tilde g_\tau(\bfx_\tau) \\
        &\leq \sum_{\tau=1}^t \Phi_\tau'(Q_\tau)\tilde g_\tau(\bfx^\star) + \frac{L^2(1+\log t)}{\mu} + 2DL\sqrt{\sum_{\tau=1}^t \Phi_\tau'(Q_\tau)^2}.
    \end{align*}
\end{lemma}
\begin{proof}
    Recall that $\nabla \tilde g_t(\bfx_t) = \nabla g_t(\bfx_t)$ if $g_t(\bfx_t)>0$, and $\bm 0$ otherwise. It follows that $\|\nabla \tilde g_t(\bfx_t)\|\leq L$. Moreover, we have that $\nabla f_t(\bfx_t) + \Phi_t'(Q_t)\nabla \tilde g_t(\bfx_t)$ is a subgradient of $f_t + \Phi_t'(Q_t)\tilde g_t$. Then, by applying the same argument as in \Cref{lem:strcvx regret}, we conclude the proof.
\end{proof}

\begin{lemma}\label{lem:Regret+Phi:adv + exp + strcvx}
    Consider the setting of \Cref{thm:adv + exp + strcvx}. Suppose that we run \Cref{alg:hard} with $\gamma_t,\Psi$ as given by \eqref{eq:psi + gamma exp}, and $\eta_t$ as given by \eqref{eq:eta strcvx}. Then, for all $t\geq 1$, we have
    \begin{align*}
        \sum_{\tau=1}^t (f_\tau(\bfx_\tau) - f_\tau(\bfx^\star))+ \Phi_{t+1}(Q_{t+1}) \leq \frac{L^2(1+\log t)}{\mu} + \log t + 1.
    \end{align*}
\end{lemma}
\begin{proof}
            We closely follow the proof of \Cref{lem:Regret+Phi:stoc + exp + cvx}. 
    By \Cref{lem:strcvx regret adv},
    \begin{align*}
            &\sum_{\tau=1}^t (f_\tau(\bfx_\tau) - f_\tau(\bfx^\star)) + \sum_{\tau=1}^t \Phi_\tau'(Q_{\tau})\tilde g_\tau(\bfx_\tau) \\
            &\leq \sum_{\tau=1}^t \Phi_\tau'(Q_\tau)\tilde g_\tau(\bfx^\star) + \frac{L^2(1+\log t)}{\mu} + 2DL\sqrt{\sum_{\tau=1}^t \Phi_\tau'(Q_\tau)^2}.
    \end{align*}
    By \Cref{lem:drift:adv + exp + cvx},
    \begin{align*}
        \Phi_{t+1}(Q_{t+1}) - \Phi_1(Q_1) \leq \sum_{\tau=1}^t \Phi_\tau'(Q_\tau)\left(\tilde g_\tau(\bfx_\tau) - \frac{R_\tau}{2}\right) + \sum_{\tau=1}^t\frac{\gamma_\tau - \gamma_{\tau+1}}{e\gamma_{\tau+1}}.
    \end{align*}
    By \Cref{lem:gamma sum exp},
    \begin{align*}
        \sum_{\tau=1}^t\frac{\gamma_\tau - \gamma_{\tau+1}}{e\gamma_{\tau+1}} \leq \frac{1+\log t}{2}
    \end{align*}
    Combining these results yields
    \begin{align*}
        \sum_{\tau=1}^t (f_\tau(\bfx_\tau)-f_\tau(\bfx^\star)) + \Phi_{t+1}(Q_{t+1}) 
        &\leq \sum_{\tau=1}^t\Phi_\tau'(Q_\tau)\tilde g_\tau(\bfx^\star)+\frac{L^2(1+\log t)}{\mu}
         + \log t + 1.
    \end{align*}
    Since we have $g_t(\bfx^\star) \leq 0$ for all $t\geq 1$ in the adversarial constraint setting, we know that $\sum_{\tau=1}^t \Phi_\tau'(Q_\tau) \tilde g_\tau(\bfx^\star) = 0$. This concludes the proof.
\end{proof}

\begin{lemma}\label{lem:R sum:adv + exp + strcvx}
    Consider the setting of \Cref{thm:adv + exp + strcvx}. Suppose that we run \Cref{alg:hard} with $\gamma_t,\Psi$ as given by \eqref{eq:psi + gamma exp}, and $\eta_t$ as given by \eqref{eq:eta strcvx}. Then, for all $t\geq 1$, we have
    \begin{align*}
        \sum_{\tau=1}^t R_\tau =  \bigO\left(\sqrt{t \log t}\right).
    \end{align*}
\end{lemma}
\begin{proof}
    We closely follow the proof of \Cref{lem:R sum:stoc + exp + strcvx}, where the only difference is that we use \Cref{lem:Regret+Phi:adv + exp + strcvx} instead of \Cref{lem:Regret+Phi:stoc + exp + cvx}. By Lipschitzness and \Cref{lem:Regret+Phi:adv + exp + strcvx}, 
    \begin{align*}
        \Phi_{t+1}'(Q_{t+1})
        &\leq DLt + \frac{L^2(1+\log t)}{\mu} + \log t + 2.
    \end{align*}
    Note that
    \begin{align*}
        \sum_{\tau=1}^t R_\tau 
        \leq \sum_{\tau=1}^t\frac{G}{\sqrt{\tau}} + \sum_{\tau=1}^t
        \frac{\Psi\left(\sum_{s=1}^\tau\Phi_s'(Q_s)^2\right)-\Psi\left(\sum_{s=1}^{\tau-1}\Phi_s'(Q_s)^2\right)}{\Phi_\tau'(Q_\tau)}.
    \end{align*}
     It is clear that $\sum_{\tau=1}^t G/\sqrt{\tau} \leq 2G\sqrt{t}$. Moreover, it follows that
     \begin{align*}
         &\sum_{\tau=1}^t\frac{\Psi\left(\sum_{s=1}^\tau\Phi_s'(Q_s)^2\right)-\Psi\left(\sum_{s=1}^{\tau-1}\Phi_s'(Q_s)^2\right)}{\Phi_\tau'(Q_\tau)}\\
         &\leq 4DL\sqrt{t}\sqrt{2\log\frac{\sum_{\tau=1}^t\Phi_\tau'(Q_\tau)}{\Phi_1'(Q_1)} +1}\\
         &\leq 4DL\sqrt{t}\sqrt{2\log\frac{t(DLt + L^2(1+\log t)/\mu +\log t + 1)}{\Phi_1'(Q_1)} +1}.
     \end{align*}
     Finally, we have $\sum_{\tau=1}^t R_\tau = \bigO\left(\sqrt{t \log t}\right).$
\end{proof}

\subsubsection{Proof of Theorem~\ref{thm:adv + exp + strcvx}}
\begin{proof}
    Under \Cref{alg:hard}, we observe that $\Violation_+(t) = \sum_{\tau=1}^t [g_\tau(\bfx_\tau)]_+ = Q_{t+1} + \sum_{\tau=1}^t R_\tau$. 
    By \Cref{lem:Regret+Phi:adv + exp + strcvx}, 
    \begin{align*}
        \Regret(t) = \sum_{\tau=1}^t (f_\tau(\bfx_\tau) - f_\tau(\bfx^\star)) \leq \frac{L^2(1+\log t)}{\mu} + \log t + 2.
    \end{align*}
    By applying the same argument in the proof of \Cref{thm:stoc + exp + strcvx}, we have
    \begin{align*}
        Q_{t+1}
        &\leq \frac{1}{\gamma_{t+1}}\log\left(DLt + \frac{L^2(1+\log t)}{\mu} + \log t + 1\right)\\
        &\leq \left(12G\sqrt{t+1} + 24DL + 1\right)\log\left(DLt + \frac{L^2(1+\log t)}{\mu} + \log t + 1\right)
    \end{align*}
    which can be written as $Q_{t+1} = \bigO\left(\sqrt{t}\log t\right)$. Moreover, by \Cref{lem:R sum:adv + exp + strcvx}, we have $\sum_{\tau=1}^t R_\tau =\bigO(\sqrt{t\log t})$. Since $\Violation(t) = Q_{t+1} + \sum_{\tau=1}^t R_\tau$, we have
    \[
        \Violation_+(t) = \bigO\left(\sqrt{t}\log t\right).
    \]
\end{proof}

\section{Auxiliary Lemmas}\label{app:aux}
\begin{lemma}\label{prop:Phi}
    For $\gamma_1, \gamma_2 >0$, define $\Phi_1, \Phi_2:\bbR\to\bbR$ as $\Phi_1(x) =e^{\gamma_1x}-1$ and $\Phi_2(x) = e^{\gamma_2x}-1$ for each $x\in \bbR$. Suppose that $\gamma_2 \leq \gamma_1$. Then we have $\Phi_2(x) - \Phi_1(x) \leq (\gamma_1 - \gamma_2)/(\gamma_2e)$ for all $x\in \bbR.$
\end{lemma}
\begin{proof}
    Note that
    \begin{align*}
        \Phi_2(x) - \Phi_1(x) 
        &= e^{\gamma_2x} - e^{\gamma_1x} \\
        &= e^{\gamma_2x}(1 - e^{(\gamma_1-\gamma_2)x}) \\
        &\leq e^{\gamma_2x} (\gamma_2 -\gamma_1)x \\
        &=e^{\gamma_2x}(-\gamma_2x) \cdot \frac{\gamma_1 - \gamma_2}{\gamma_2} \\
        &\leq \frac{\gamma_1 - \gamma_2}{e\gamma_2}
    \end{align*}
    where the first inequality follows from $1-e^{-y} \leq y$ for all $y \in \bbR$, and the second inequality follows from the fact that $e^{-y}y \leq 1/e$ for all $y \in \bbR$ and that $(\gamma_1 - \gamma_2)/\gamma_2 \geq 0$.
\end{proof}

\begin{lemma}\label{prop:psi}
    Let $\delta\in (0,1)$. Define $\psi:\bbR_+ \to \bbR$ as $\psi(x) = (1+x)\log\frac{(\pi^2/6)(1 + \log_2 (1+x))^2}{\delta}$ for each $x \in \bbR_+$. Suppose that $0 \leq x_1 < x_2$. Then we have
    \begin{align*}
        \sqrt{\psi(x_2)} - \sqrt{\psi(x_1)} \leq 8\sqrt{\log\frac{(\pi^2/6)(1+\log_2(1+x_2))^2}{\delta}} \left(\sqrt{1+x_2} - \sqrt{1+x_1}\right).
    \end{align*}
    Moreover, we have
    \begin{align*}
        \frac{\sqrt{\psi(x_2)} - \sqrt{\psi(x_1)}}{\sqrt{x_2 - x_1}} \leq 8\sqrt{\log\frac{(\pi^2/6)(1+\log_2(1+x_2))^2}{\delta}}.
    \end{align*}
\end{lemma}
\begin{proof}
    For $x \in [x_1,x_2]$, we observe that
    \begin{align*}
        &\frac{d}{dx}\sqrt{\psi(x)} \\
        &= \frac{1}{2\sqrt{x+1}}\sqrt{\log\frac{(\pi^2/6)(1+\log_2(1+x))^2}{\delta}} \\
        &\quad+ \frac{1}{\log2\cdot\sqrt{1+x} (1+\log_2(1+x))\sqrt{\log\frac{(\pi^2/6)(1+\log_2(1+x))^2}{\delta}}} \\
        &\leq \frac{1}{2\sqrt{x+1}}\sqrt{\log\frac{(\pi^2/6)(1+\log_2(1+x))^2}{\delta}} + \frac{1}{\log2\cdot\sqrt{\log(\pi^2/6)} \sqrt{x+1}}\\
        &\leq \frac{4}{\sqrt{x+1}}\sqrt{\log\frac{(\pi^2/6)(1+\log_2(1+x))^2}{\delta}}\\
        &\leq \frac{4}{\sqrt{x+1}}\sqrt{\log\frac{(\pi^2/6)(1+\log_2(1+x_2))^2}{\delta}}.
    \end{align*}
    Then it follows that
    \begin{align*}
        \sqrt{\psi(x_2)} - \sqrt{\psi(x_1)} 
        &\leq \int_{x_1}^{x_2}\frac{d}{dx}\sqrt{\psi(x)}dx\\
        &\leq 8\sqrt{\log\frac{(\pi^2/6)(1+\log_2(1+x_2))^2}{\delta}} \left(\sqrt{1+x_2} - \sqrt{1+x_1}\right).
    \end{align*}
    Note that $\sqrt{1+x_2} - \sqrt{1+x_1} \leq \sqrt{x_2 - x_1}$. This leads to the second statement.
\end{proof}

\begin{lemma}\label{lem:integral 1/x}
    Let $\{z_t\}_{t\geq 0}$ be an increasing sequence, and let $z_0 > 0$. Then, we have $\sum_{\tau=1}^t(z_\tau - z_{\tau-1})/z_\tau \leq \log\left(z_t/z_0\right)$ for all $t\geq 1$.
\end{lemma}
\begin{proof} Note that $\sum_{\tau=1}^t (z_\tau - z_{\tau-1})/z_\tau \leq \int_{z_0}^{z_t} (1/x) dx = \log\left(z_t/z_0\right)$.
\end{proof}

The following lemma states the maximal inequality for supermartingales (Theorem 3.9 in \cite{lattimore2020bandit}), which can be viewed as a Markov-type inequality for the maximum of a supermartingale. 
\begin{lemma}\label{lem:maximal ineq}
    Let $\{X_t\}_{t\geq 0}$ be a supermartingale with $X_t \geq 0$ almost surely for all $t \geq 0$. Then for every $\epsilon > 0$,
    \[
        \bbP\left(\sup_{t \geq 0} X_t \geq \epsilon\right) \leq \frac{\bbE[X_0]}{\epsilon}.
    \]
\end{lemma}

The following lemma is a conditional variant of Hoeffding's lemma (Lemma B.7 in \cite{shalev2014understanding}).
\begin{lemma}\label{lem:hoeffding lemma}
    Let $X$ be a random variable, and let $\calF$ be a $\sigma$-algebra. Suppose that $\bbE[X|\calF] \leq 0$ and $|X| \leq Y$ almost surely, where $Y$ is a positive $\calF$-measurable random variable. Then, for every $\lambda > 0$, almost surely,
    \[
        \bbE[e^{\lambda X}|\calF] \leq e^{\frac{\lambda^2 Y^2}{2}}.
    \]
\end{lemma}
\begin{proof}
We closely follow the proof of Lemma B.7 in \cite{shalev2014understanding}. Fix $\lambda >0$. From the proof of Lemma B.7 in \cite{shalev2014understanding}, for any $y >0$ and $x \in [-y, y]$
\begin{align*}
    e^{\lambda x}\leq \frac{\left ( y-x \right )e^{-\lambda y}+(x+y)e^{\lambda y}}{2y}.
\end{align*}
Since we have $|X|\leq Y$ almost surely and $Y$ is $\calF$-measurable, by letting $x \leftarrow X,\ y \leftarrow Y$ and taking $\bbE[\cdot|\calF]$ on both sides, we have
\begin{align*}
    \mathbb{E}[e^{\lambda X}|\mathcal{F}]\leq \frac{\left ( Y-\mathbb{E}\left[X|\calF\right] \right )e^{-\lambda Y}+(\bbE[X|\calF]+Y)e^{\lambda Y}}{2Y} \leq \frac{e^{-\lambda Y}+e^{\lambda Y}}{2}
\end{align*}
where the second inequality follows from that $\bbE[X|\calF] \leq 0$ and $Y > 0$ imply $\bbE[X|\calF](e^{\lambda Y}-e^{-\lambda Y})/(2Y) \leq 0$. Again, from the proof of Lemma B.7 in \cite{shalev2014understanding}, we know that $(e^{-x}+e^{x})/2 \leq e^{x^2/2}$ for all $x\in \bbR$. This completes the proof.
\end{proof}

The following lemma states standard properties of strongly convex functions.
\begin{lemma}\label{lem:strcvx properties}
    Let $f:\bbR^d\to (-\infty, \infty]$ be a proper convex function. Suppose that $f$ is $\mu$-strongly convex for some $\mu > 0$. We have for all $\bfx \in \ridom(f), \bfy \in \dom(f)$, $\nabla f(\bfx) \in \partial f(\bfx)$,
    \[
        f(\bfy) \geq f(\bfx)+ \nabla f(\bfx)^\top (\bfy -\bfx) + \frac{\mu}{2}\|\bfy-\bfx\|_2^2.
    \]
    Moreover, let $g:\bbR^d\to (-\infty, \infty]$ be a proper convex function. Then, $f+g$ is $\mu$-strongly convex.
\end{lemma}
\begin{proof}
    The first statement follows from Theorem 5.24 in \cite{beck2017first} and the fact that a proper convex function $f$ is subdifferentiable in $\ridom(f)$ (Theorem 23.4 in \cite{rockafellar1970convex}).
    The second statement follows from Lemma 5.20 in \cite{beck2017first}. 
\end{proof}

\bibliographystyle{plainnat}
\bibliography{ref}

\end{document}